\documentclass[11pt,a4paper]{article}
\usepackage[a4paper, portrait, margin=1.1811in]{geometry}
\usepackage[english]{babel}

\usepackage{amsmath,amsfonts,amssymb}
\usepackage{algorithmic,algorithm}
\usepackage{array}
\usepackage{textcomp}
\usepackage{multicol}
\usepackage{url,verbatim}
\usepackage{graphicx,float}
\usepackage{amsthm}
\usepackage{authblk}
\usepackage{setspace}

\usepackage{subcaption}
\usepackage{xcolor}
\usepackage[colorlinks=true, allcolors=blue]{hyperref}
\usepackage{bm}
\usepackage{tikz}
\usetikzlibrary{graphs, graphs.standard}
\usepackage{optidef}

\usepackage{lineno}
\modulolinenumbers[5]

\newtheorem{definition}{Definition}
\newtheorem{theorem}{Theorem}
\newtheorem{proposition}{Proposition}
\newtheorem{remark}{Remark}

\begin{document}


\setlength{\parskip}{12pt} 
\setlength{\parindent}{0pt} 
\onehalfspacing  
\setlength{\intextsep}{20pt} 

\mathtoolsset{showonlyrefs,showmanualtags} 


\title{Graph-Based Correlation Matrix Generation: A Convex Optimization Approach}

\author[1]{Ali Fakhar\thanks{Corresponding author: ali.fakhar@univ-grenoble-alpes.fr}}
\author[1]{Kévin Polisano}
\author[2]{Irène Gannaz}
\author[1]{Sophie Achard}
\affil[1]{Univ. Grenoble Alpes, CNRS, Grenoble INP, Inria, LJK, F-38000 Grenoble, France}
\affil[2]{Univ. Grenoble Alpes, CNRS, Grenoble INP, G-SCOP, 38000 Grenoble, France}
\date{}

\maketitle

\begin{abstract}
This work addresses the generation of theoretical correlation matrices with prescribed sparsity patterns associated to graph structures. We propose a novel convex optimization framework in which an initial matrix is projected onto an elliptope under a positive semidefiniteness constraint. Several numerical schemes are implemented and compared. The problem falls within the broader class of matrix completion, where off-diagonal entries corresponding to absent edges are fixed to zero and diagonal entries are fixed to one. Beyond this structural constraint, the approach offers greater flexibility than existing methods by allowing control over the mean of the off-diagonal entry distribution, enabling the generation of correlation matrices that better reflect realistic data. This procedure is not designed to yield a uniform distribution over the feasible set; rather, it provides a principled and tunable way to construct correlation matrices suitable for benchmarking statistical methods for graphical model inference. Theoretical guarantees on the existence of solutions are established, both in the general setting and under the additional mean constraint. Simulation studies illustrate the properties of the generated matrices with respect to graph structure. The methodology is applied to two real-world datasets from neuroscience and finance, and a comparison with GAN-based correlation matrix generation is provided.
\end{abstract}


\bigskip

\section{Introduction}
\label{sec:introduction}

Graphical models provide a principled framework for representing and
inferring dependencies among random variables, with applications spanning
genetics \cite{grechkin2015pathway}, proteomics \cite{akbani2014pan},
disease characterization \cite{armour2017network}, functional brain
connectivity \cite{huang2010learning}, and risk management \cite{hull2012risk}.
A central object in this framework is the correlation matrix (or its inverse,
the precision matrix). Assessing the quality of
graph inference procedures requires simulation studies, which in turn
necessitate generating \emph{theoretical} correlation matrices with a
prescribed zero pattern, a task that is non-trivial in general.

Several methods have been proposed to generate structured correlation
matrices. Classical approaches such as the vines and onion procedures
\cite{lewandowski2009generating}, building on the characterization of
\cite{joe2006generating}, produce matrices whose entries are uniformly
distributed over the set of correlation matrices, but do not accommodate
structural (graph-based) constraints. Methods based on the Cholesky
decomposition \cite{pourahmadi2015distribution, cordoba2020generating}
can incorporate such constraints, but are restricted to \emph{chordal}
graphs, which inherently admit a Perfect Elimination Ordering (PEO). To handle
non-chordal structures, \cite{cordoba2020generating} introduced diagonal
dominance and partial orthogonalization approaches; however, both methods
suffer from significant limitations, the former concentrates generated
entries near zero, while the latter depends on the arbitrary labeling of
graph nodes.

Moreover, a consistent empirical observation in neuroscience \cite{achard2022generation} is that brain connectivity correlations exhibit a positive shift, whereas existing generation methods produce matrices with entry distributions centered around zero. Applying a similar perspective to financial market data via Generative Adversarial Networks (GANs) \cite{Gautier2020} invites a deeper exploration of these distributional characteristics. This contrast between synthetic and real-world correlation matrices poses a fundamental challenge for the validity of realistic simulation studies.

In this paper, we propose a novel approach for generating structured
correlation matrices based on convex optimization. Given an initial matrix
$\bar{\mathbf{C}}$ and a graph $G$, we formulate the problem as the
projection of $\bar{\mathbf{C}}$ onto the set $\mathcal{C}(G)$ of
correlation matrices compatible with $G$, in the Frobenius sense. This
formulation is applicable to \emph{arbitrary} graph structures, including
non-chordal graphs, and yields a unique, node-ordering-independent solution.
Crucially, by introducing an additional linear inequality constraint on the
mean of the generated entries, our method can produce correlation matrices
whose distributional properties match those of empirical data, such as the
positive shift observed in functional brain connectivity or financal data.
We present three
solvers for the resulting convex program --- a dual approach \cite{malick2004dual},
a general-purpose convex solver via CVXPY \cite{diamond2016cvxpy}, and the
specialized quadratic semidefinite solver QSDPNAL \cite{li2018qsdpnal} ---
and compare their computational performance across problem dimensions ranging
from $p = 51$ to $p = 1{,}000$ vertices.

The remainder of this paper is organized as follows.
Section~\ref{sec:preliminaries} establishes the preliminaries and foundational
definitions for graph-structured correlation matrices.
Section~\ref{sec:state_of_the_art} reviews existing state-of-the-art generation
methods and details their limitations.
Section~\ref{sec:framework} introduces our proposed convex optimization framework
and discusses its theoretical properties.
Section~\ref{sec:solvers} outlines the three algorithms used to solve the
optimization problem along with numerical post-processing.
Section~\ref{sec:baseline_experiments} presents baseline numerical experiments on
the geometry of the elliptope, initialization effects, graph structure effects,
and solver benchmarks.
Section~\ref{sec:mean_constraint} details the inclusion of the mean constraint to
control entry distributions, including theoretical bounds and feasibility analysis.
Section~\ref{sec:applications} provides real-world applications and experiments.
Finally, Section~\ref{sec:conclusion} concludes the paper.


\section{Preliminaries and Problem Formulation}
\label{sec:preliminaries}

From a statistical perspective, consider a random vector $\bm{x} \in \mathbb{R}^p$
with covariance matrix
$\mathbf{\Sigma} = \mathbb{E}[(\bm{x} - \mathbb{E}[\bm{x}])(\bm{x} -
\mathbb{E}[\bm{x}])^{\top}]$.
The covariance matrix $\mathbf{\Sigma}$ captures the linear dependencies between
the components of $\bm{x}$. The corresponding correlation matrix $\mathbf{C}$ is
obtained by normalizing the covariance matrix,
\begin{equation}
\mathbf{C} = (\text{diag}(\mathbf{\Sigma}))^{-1/2} \mathbf{\Sigma}
(\text{diag}(\mathbf{\Sigma}))^{-1/2},
\label{eq:corr_from_cov}
\end{equation}
where $\text{diag}(\mathbf{\Sigma})$ denotes the diagonal matrix formed by the
diagonal elements of $\mathbf{\Sigma}$. This transformation standardizes each
variable to have unit variance, making the correlation coefficients
$c_{ij} = \frac{\sigma_{ij}}{\sqrt{\sigma_{ii}\sigma_{jj}}}$ scale-invariant and
directly comparable.

A key property of covariance and correlation matrices is their positive
semidefiniteness. A real symmetric $p \times p$ matrix $\mathbf{A}$ is positive
semidefinite (PSD), denoted $\mathbf{A} \succeq 0$, if
$\bm{v}^{\top}\mathbf{A}\bm{v} \geq 0$ for all $\bm{v} \in \mathbb{R}^p$, and
positive definite (PD), denoted $\mathbf{A} \succ 0$, if
$\bm{v}^{\top}\mathbf{A}\bm{v} > 0$ for all non-zero $\bm{v} \in \mathbb{R}^p$.
The PSD property of correlation matrices is inherited from the covariance matrix,
as for any vector $\bm{v} \in \mathbb{R}^p$, we have
\begin{equation*}
\bm{v}^{\top}\mathbf{\Sigma}\bm{v} =
\mathbb{E}[(\bm{v}^{\top}(\bm{x} - \mathbb{E}[\bm{x}]))^2] \geq 0.
\end{equation*}
In the context of generating correlation matrices with prescribed zero patterns,
we work with the broader class of PSD matrices ($\mathbf{\Sigma} \succeq 0$) to
ensure feasibility of the optimization problem.

Let us denote by $\mathcal{C}$ the set of correlation matrices in
$\mathbb{R}^{p \times p}$. A matrix $\mathbf{C} \in \mathcal{C}$ is a symmetric
PSD matrix satisfying
\begin{equation}
\text{diag}(\mathbf{C}) = \mathbf{1}, \quad \text{and} \quad
-1 \leq c_{ij} \leq 1 \quad \forall i, j \in \{1, 2, \ldots, p\}.
\label{eq:corr_constraints}
\end{equation}

A graph $G = (V, E)$ is a mathematical structure representing pairwise relations
between objects, where $V = \{v_1, v_2, \ldots, v_p\}$ is the set of vertices
and $E \subset V \times V$ is the set of edges. An edge $(i,j) \in E$ indicates
that vertices $i$ and $j$ are connected. We are interested in correlation matrices
that are associated with a given graph structure $G$. Specifically, we seek
correlation matrices $\mathbf{C}$ such that $c_{ij} = 0$ if $(i,j) \notin E$
for all $i \neq j$. We denote by $\mathcal{C}(G)$ the set of correlation matrices
associated with graph $G$:
\begin{equation}
\begin{split}
\mathcal{C}(G) = \big\{ \mathbf{C} = (c_{ij}) : \;
&\mathbf{C} \text{ is PSD, satisfies } \eqref{eq:corr_constraints}, \\
&\text{and } c_{ij} = 0 \text{ for all } (i,j) \notin E \big\}.
\end{split}
\label{eq:omega_G}
\end{equation}

In this paper, our primary objective is the generation of PSD matrices under a
prescribed zero pattern. As such, our framework is general and not exclusively
restricted to either correlation or precision matrices. However, we briefly
clarify their distinct statistical interpretations: under a Gaussian distribution,
a zero in the correlation matrix indicates marginal independence between two
variables~\cite[Chapter 5]{lauritzen1996graphical}, while a zero in the precision matrix indicates conditional independence.
Conversely, for non-Gaussian distributions, these zeros imply marginal
uncorrelatedness and conditional uncorrelatedness, respectively, which do not
strictly guarantee statistical independence~\cite{baba2004partial}.

This formulation places our problem within the broader framework of \emph{matrix
completion}. Classically, the positive semidefinite (or positive definite) matrix
completion problem asks whether a \emph{partially specified} symmetric matrix, with entries fixed on a set $E$ and left \emph{free} (unspecified) on its
complement, can be completed to a full PSD matrix by an appropriate choice of
the free entries~\cite{grone1984positive}. A classical result in this literature~\cite{grone1984positive}
states that every partial matrix with specified diagonal and specified entries on
the edges of a chordal graph $G$, whose every fully-specified principal
submatrix is PSD, admits a PSD completion.
Our setting can be viewed as a special case of this framework in which the roles
of $E$ and its complement are reversed relative to the classical completion
problem. Rather than leaving the entries on $E^c$ (non-edges) \emph{free} and
asking for feasibility, we \emph{fix} them to zero as part of the problem
data, together with the unit-diagonal constraint, and treat the entries on $E$
as the variables.

\section{State-of-the-Art Methods}
\label{sec:state_of_the_art}

In this section, we provide a critical review of existing methods available for
generating correlation matrices with prescribed graph structures.

To the best of our knowledge, three principal methods have been proposed in the
literature for generating correlation matrices under a prescribed zero pattern. These methods are known as a \textit{Cholesky decomposition} approach~\cite{cordoba2018metropolis},
\textit{partial orthogonalization}, and \textit{diagonal dominance}, the latter
two both introduced in~\cite{cordoba2020generating}. The Cholesky-based method
constructs $\mathbf{C} = \mathbf{L}\mathbf{L}^\top$ from a lower triangular
factor $\mathbf{L}$, and is restricted to \emph{chordal} graphs. The zero
pattern of $\mathbf{C}$ is preserved by the Cholesky factorization if and only
if the associated graph admits a Perfect Elimination Ordering (PEO), which holds
precisely when the graph is chordal~\cite{vandenberghe2015chordal}. Moreover, the same correspondence between the zero pattern of $\mathbf{C}$
and that of its Cholesky factor $\mathbf{L}$ makes it possible to sample the
free entries of $\mathbf{L}$ so as to achieve a uniform sampling over
$\mathcal{C}(G)$ for chordal graphs~\cite{cordoba2020generating}. Intuitively,
a PEO allows the vertices to be eliminated one at a time such that each
eliminated vertex's remaining neighbors form a clique; this ordering lets the
Cholesky factorization proceed without ever generating fill-in entries at
positions corresponding to non-edges of $G$, so the prescribed zeros are exactly
reproduced. For non-chordal graphs, no such ordering exists, fill-in is
unavoidable, and the resulting matrix violates the required zero pattern. The
remaining two methods, diagonal dominance and partial orthogonalization, do not
rely on a PEO and are therefore applicable to \emph{arbitrary} graph structures,
including non-chordal ones. Since this paper targets the generation of
correlation matrices under general, possibly non-chordal zero patterns, we
restrict our detailed review and experimental comparison to these latter two
methods. We examine their algorithmic formulations below and highlight their
limitations across distributional properties, node-ordering dependency, and
structural restrictions.

When considering criteria for uniformity, all existing methods, to the best of our knowledge, fail to generate samples uniformly within the feasible space $\mathcal{C}(G)$. It is crucial, however, to distinguish between uniformity of sampling, which pertains to the selection of matrices according to the Lebesgue measure over the manifold $\mathcal{C}(G)$, and the uniformity of the resulting entry distributions. A uniform sampling of correlation matrices over $\mathcal{C}(G)$ does not imply that the individual entries $c_{ij}$ are uniformly distributed in $[-1, 1]$. Because every valid correlation matrix must be positive semidefinite (PSD) and satisfy the unit-diagonal constraint $\text{diag}(\mathbf{C}) = \bm{1}$, the entries are confined to a compact convex body known as the elliptope. Within this space, geometric constraints inherently force the distribution of entries to be non-uniform. Consequently, even if a method were to achieve perfect uniform sampling over the feasible space, the marginal distributions of its entries would still exhibit significant bias as a result of the underlying geometry of $\mathcal{C}(G)$. Indeed, as demonstrated by Joe \cite{joe2006generating}, even under the complete graph assumption, the marginal distribution of an off-diagonal entry $c_{ij}$ follows a symmetric beta distribution on the interval $[-1, 1]$, rather than a uniform one.

\subsection{Diagonal Dominance}
\label{subsec:diagonal_dominance}

The diagonal dominance approach, detailed in~\cite{cordoba2020generating},
constructs a positive definite matrix by first sampling off-diagonal entries
uniformly from $[-1, 1]$, subject to the prescribed zero pattern, and then
enforcing diagonal dominance via the update rule:
\begin{equation}
    \forall i \in \{1, \ldots, p\}, \quad \tilde{c}_{ii} \leftarrow
    \sum_{\substack{j=1 \\ j \neq i}}^{p} |\tilde{c}_{ij}| + \delta_i,
\end{equation}
where $\delta_i > 0$ is a small positive random perturbation ensuring strict
positive definiteness (omitting $\delta_i$ yields a positive semidefinite matrix).
The resulting matrix is then normalized to recover a valid correlation matrix:
\begin{equation}
    \mathbf{C} =
    \mathrm{diag}(\tilde{\mathbf{C}})^{-1/2}\,\tilde{\mathbf{C}}\,
    \mathrm{diag}(\tilde{\mathbf{C}})^{-1/2}.
\end{equation}

While computationally inexpensive and applicable to arbitrary graph structures,
diagonal dominance suffers from a critical limitation. The normalization step
systematically shrinks off-diagonal entries toward zero. As illustrated later in
our numerical experiments (Figure~\ref{fig:density_comparison}), the resulting distribution of non-diagonal, non-zero
entries is heavily concentrated near zero, making it poorly suited for simulating
datasets where correlations are moderate to strong.

This behavior arises because the diagonal entries grow proportionally to the sum
of absolute off-diagonal values, causing the ratio
$\tilde{c}_{ij} / \sqrt{\tilde{c}_{ii}\tilde{c}_{jj}}$ to remain small regardless
of the initial magnitude of $\tilde{c}_{ij}$. The bimodal shape of the diagonal
dominance distribution is a consequence of the saturation of the nonlinear
normalization map. An expression of the density is provided in \ref{app:bimodal}.

\subsection{Partial Orthogonalization}
\label{subsec:partial_orthogonalization}

The partial orthogonalization method, proposed in~\cite{cordoba2020generating},
offers an alternative that is applicable to non-chordal graphs. Starting from an
initial matrix $\mathbf{C}_0 = \mathbf{Q}_0 \mathbf{Q}_0^\top$ that satisfies a
set of the desired zero pattern (obtained, for instance, by changing non-chordal graph to chordal graph with adding edges, such that the
graph $G$ forms a chordal graph and applying a Cholesky-based
procedure~\cite{cordoba2018metropolis}), the algorithm iteratively orthogonalizes
rows of $\mathbf{Q}_0$ with respect to specific subsets of other rows.

Formally, for each row vector $\bm{q}_{i\cdot}$, the update enforces:
\begin{equation}
    \bm{q}_{i\cdot} \leftarrow \bm{q}_{i\cdot} -
    \sum_{\substack{j < i \\ (i,j) \notin E}}
    \frac{\langle \bm{q}_{i\cdot}, \bm{q}_{j\cdot} \rangle}
    {\|\bm{q}_{j\cdot}\|^2} \bm{q}_{j\cdot},
\end{equation}
using a modified Gram-Schmidt procedure. After all rows have been processed,
the factor matrix $\mathbf{Q}$ yields
$\mathbf{C} = \mathbf{Q}\mathbf{Q}^\top \in \mathcal{C}(G)$.

Although partial orthogonalization produces correlation matrices with a broader
distribution of off-diagonal values than diagonal dominance, the procedure heavily depends on the node labeling of the graph.
Different orderings of vertices lead to different output distributions, meaning it
lacks node-ordering independence.

More generally, none of the state-of-the-art methods offers control over the
moments of the entry distribution, for instance, none can enforce a shift
of the mean toward positive values.
\section{Proposed Convex Optimization Framework}
\label{sec:framework}

To generate structured correlation matrices without suffering from near-zero
shrinkage or node-labeling dependence, we formulate the matrix generation task as
a matrix projection problem. Given an initial matrix $\bar{\mathbf{C}}$, we find its
closest graph-compatible representation by solving the following optimization
problem:
\begin{mini}
{\mathbf{C}}{\frac{1}{2} \|\mathbf{C} - \bar{\mathbf{C}}\|_F^2}
{}{}
\addConstraint{\mathbf{C}}{\in \mathcal{C}(G),}
\label{eq:optimization_problem}
\end{mini}
where $\|\cdot\|_F$ denotes the Frobenius norm.

The choice of the squared Frobenius norm as the objective function is motivated
by its leading to a quadratic semidefinite optimization problem, which exhibits
strong convergence properties compared to linear formulations \cite{boyd2004convex}.
Moreover, this projection framework directly allows practitioners to guide the
generation process by supplying an empirical or specialized matrix $\bar{\mathbf{C}}$
to inject desired global properties.

It is important to note that the element-wise bound constraint $|c_{ij}| \leq 1$
specified in \eqref{eq:corr_constraints} is implicitly guaranteed by the positive
semidefiniteness constraint. A symmetric matrix $\mathbf{C}$ is PSD if and only if
all its principal minors are non-negative~\cite{meyer2000matrix}. For any
$2 \times 2$ principal submatrix 
$\begin{pmatrix}
1 & c_{ij} \\
c_{ij} & 1
\end{pmatrix},$
positive semidefiniteness strictly requires its determinant to be non-negative
($1 - c_{ij}^2 \geq 0$), which simplifies to $|c_{ij}| \leq 1$. Therefore,
explicit boundary constraints are redundant and can be omitted.

The optimization problem \eqref{eq:optimization_problem} satisfies the conditions
necessary to guarantee a unique solution:
\begin{enumerate}
    \item \textbf{Strict Convexity:} The objective function
    $f(\mathbf{C}) = \frac{1}{2}\|\mathbf{C} - \bar{\mathbf{C}}\|_F^2$ is strictly
    convex with respect to $\mathbf{C}$.
    \item \textbf{Set Convexity:} The constraint set $\mathcal{C}(G)$ is convex.
    The set of positive semidefinite matrices forms a convex cone, the unit diagonal
    constraint forms an affine equality hyperplane, and the structural zero patterns
    $c_{ij} = 0$ for $(i,j) \notin E$ are linear equality constraints. The
    intersection of these closed convex sets is convex.
    \item \textbf{Non-Emptiness:} The feasible set is non-empty because the identity
    matrix $\mathbf{I}$ trivially satisfies all constraints for any graph $G$
    ($\mathbf{I} \succeq 0$, $\text{diag}(\mathbf{I})=\mathbf{1}$, and all
    off-diagonal entries are 0, satisfying the zero condition).
\end{enumerate}

By standard principles of convex analysis \cite{boyd2004convex}, a strictly convex
objective function over a non-empty, closed convex feasible set admits a unique
global minimizer. Therefore, problem \eqref{eq:optimization_problem} always yields
a unique solution, eliminating any dependence on node ordering or algorithmic
initialization trajectories.

Because of uniqueness, the notion of initialization in our framework refers not to the starting point of an iterative algorithm, but to the choice of $\bar{\mathbf{C}}$ itself, which fully determines the output $\hat{\mathbf{C}}$. Two choices are considered in this work:
\begin{itemize}
    \item \textbf{Uniform random initialization:} The natural choice for 
    $\bar{\mathbf{C}}$ would be to sample it uniformly with respect to 
    $\mathcal{C}(G)$ as it is addressed in \cite{pourahmadi2015distribution, lewandowski2009generating}. However, as we will show, uniformity is not preserved 
    under our projection mapping even under this initialization. For simplicity, 
    we therefore sample the lower triangular (and diagonal) entries of 
    $\bar{\mathbf{C}}$ independently and uniformly from $[-1, 1]$, with the 
    upper triangular entries set by symmetry. This offers coverage of the full hypercube $[-1,1]^{p(p-1)/2}$ in the space of 
    symmetric matrices $\mathcal{S}^p$, prior to projection onto $\mathcal{C}(G)$.

    \item \textbf{Empirical initialization:} $\bar{\mathbf{C}}$ is set to the
    empirical correlation matrix computed from observed data. In this case,
    problem \eqref{eq:optimization_problem} finds the closest valid correlation
    matrix in $\mathcal{C}(G)$ to the observed one in the Frobenius sense,
    preserving the global statistical characteristics of the initial matrix.
\end{itemize}

More generally, the mapping
$\bar{\mathbf{C}} \mapsto \hat{\mathbf{C}} = \mathrm{proj}_{\mathcal{C}(G)}
(\bar{\mathbf{C}})$ is a continuous, nonexpansive map (being a projection onto a
closed convex set). Therefore, the distribution of the output $\hat{\mathbf{C}}$
varies smoothly with that of $\bar{\mathbf{C}}$, i.e. two initial
matrices close in
Frobenius norm produce output matrices that are also close.

Dealing with the strictness of the zero pattern, ideally, one might desire a
formulation where $c_{ij} = 0$ if and only if $(i,j) \notin E$. However, our
problem formulation enforces $c_{ij} = 0$ if $(i,j) \notin E$, which means that
while non-edges strictly guarantee zero entries, we may occasionally observe
$c_{ij} = 0$ for some edges $(i,j) \in E$. We refrain from imposing strict non-zero constraints ($c_{ij} \neq 0$) for
edges, as this would break convexity of the feasible domain; see
\ref{app:measure} for a formal justification.

\section{Three Methods for Solving the Optimization Problem}
\label{sec:solvers}

We now present three different numerical approaches to solve the convex
optimization problem \eqref{eq:optimization_problem}. Each method offers distinct
trade-offs in terms of computational scalability, ease of implementation, and
algorithmic architecture.

\subsection{Dual Approach}
\label{subsec:malick_dual}

The dual approach, adapted from Malick \cite{malick2004dual}, transforms the
constrained semidefinite least-squares problem into an unconstrained
optimization problem over a lower-dimensional dual space.

\subsubsection{Theoretical Framework}
Consider a general semidefinite least-squares problem:
\begin{mini}|l|
{\mathbf{X}}{\frac{1}{2} \left\| \mathbf{X} - \bar{\mathbf{X}} \right\|_F^2}
{\label{eq:sdls_general}}{}
\addConstraint{\mathbf{X}}{\succeq 0}
\addConstraint{\mathcal{A}\mathbf{X}}{= \bm{\alpha},}
\end{mini}
where $\mathcal{A}: \mathcal{S}^p \to \mathbb{R}^m$ is a linear operator mapping
the space of $p \times p$ symmetric matrices $\mathcal{S}^p$ to the constraint
space $\mathbb{R}^m$. The Lagrangian associated with \eqref{eq:sdls_general} is
defined as:
\begin{equation}
\mathcal{L}(\mathbf{X}, \mathbf{Y}, \bm{y}) =
\frac{1}{2}\|\mathbf{X} - \bar{\mathbf{X}}\|_F^2
- \langle \mathbf{Y}, \mathbf{X} \rangle_F
+ \langle \bm{y}, \mathcal{A}\mathbf{X} - \bm{\alpha} \rangle_{\mathbb{R}^m},
\label{eq:lagrangian}
\end{equation}
where $\mathbf{Y} \in \mathcal{S}^p$ ($\mathbf{Y} \succeq 0$) and
$\bm{y} \in \mathbb{R}^m$ represent the dual variables. Here,
$\langle \cdot, \cdot \rangle_F$ denotes the Frobenius inner product
$\langle \mathbf{A}, \mathbf{B} \rangle_F = \text{Tr}(\mathbf{A}\mathbf{B})$,
and $\langle \cdot, \cdot \rangle_{\mathbb{R}^m}$ represents the standard
Euclidean inner product on $\mathbb{R}^m$. The adjoint operator
$\mathcal{A}^*: \mathbb{R}^m \to \mathcal{S}^p$ is uniquely determined via the
duality relation:
\begin{equation}
\langle \mathcal{A}\mathbf{X}, \bm{y} \rangle_{\mathbb{R}^m} =
\langle \mathbf{X}, \mathcal{A}^{*}\bm{y} \rangle_F,
\quad \forall \mathbf{X} \in \mathcal{S}^p, \; \bm{y} \in \mathbb{R}^m.
\label{eq:adjoint_def}
\end{equation}
By setting the first-order optimality condition with respect to $\mathbf{X}$ to
zero, we obtain the primal-dual relation
$\mathbf{X} = \bar{\mathbf{X}} + \mathbf{Y} - \mathcal{A}^* \bm{y}$.

\subsubsection{Application to Our Projection Problem}
In our context, the matrix $\mathbf{X}$ corresponds to the graph-compatible
correlation matrix $\mathbf{C}$, and $\bar{\mathbf{X}}$ is the initial
matrix $\bar{\mathbf{C}}$. The total number of linear equality constraints is
$m = p + q$, where $p$ is the number of diagonal elements and $q$ is the number
of off-diagonal entries forced to zero. Let the index pairs
$(i_1, j_1), (i_2, j_2), \ldots, (i_q, j_q)$ denote the explicit set of
non-edges in the graph, satisfying $(i_k, j_k) \notin E$ for all
$k \in \{1, \ldots, q\}$.

Let $\mathbf{E}_{ij}$ denote the symmetric elementary matrix with ones at indices
$(i,j)$ and $(j,i)$ and zeros elsewhere. The linear operator $\mathcal{A}$
evaluates the matrix entries corresponding to both the diagonal and non-edge
constraints:
\begin{equation}
 \mathcal{A}\mathbf{C} =
 \begin{pmatrix}
   \langle \mathbf{E}_{11}, \mathbf{C} \rangle_F \\ \vdots \\
   \langle \mathbf{E}_{pp}, \mathbf{C} \rangle_F \\
   \langle \mathbf{E}_{i_1j_1}, \mathbf{C} \rangle_F \\ \vdots \\
   \langle \mathbf{E}_{i_qj_q}, \mathbf{C} \rangle_F
 \end{pmatrix}
 =
 \begin{pmatrix} 1 \\ \vdots \\ 1 \\ 0 \\ \vdots \\ 0 \end{pmatrix}
 = \bm{\alpha}.
\end{equation}
The dual vector $\bm{y} \in \mathbb{R}^m$ can be partitioned into
$\bm{v} \in \mathbb{R}^p$ (the multipliers for the unit diagonal) and
$\bm{w} \in \mathbb{R}^q$ (the multipliers for the graph structural zeros).
Utilizing the adjoint definition, the operator maps these dual variables back
into matrix space:
\begin{equation}
\mathcal{A}^*(\bm{y}) =
\sum_{k=1}^p v_k \mathbf{E}_{kk} + \sum_{k=1}^q w_k \mathbf{E}_{i_k j_k}
= \mathrm{diag}(\bm{v}) + \mathbf{Z},
\label{eq:adjoint_specific}
\end{equation}
where $\mathbf{Z}$ is a symmetric sparse matrix containing the dual variables
$\bm{w}$ at the structural non-edge positions. The dual problem simplifies to:
\begin{maxi}|l|
{\bm{y}}{-\frac{1}{2}\|\mathrm{proj}_{\succeq 0}
(\bar{\mathbf{C}} - \mathcal{A}^*\bm{y})\|_F^2
+ \frac{1}{2}\|\bar{\mathbf{C}}\|_F^2
- \langle \bm{y}, \bm{\alpha} \rangle_{\mathbb{R}^m},}
{\label{eq:dual_our_problem}}{}
\end{maxi}
where $\mathrm{proj}_{\succeq 0}(\cdot)$ denotes the orthogonal projection onto
the positive semidefinite cone, computed via eigenvalue decomposition. We solve
this low-dimensional, unconstrained dual problem using the L-BFGS-B
algorithm~\cite{byrd1995limited} and reconstruct the unique primal optimal matrix
via $\mathbf{C}^* = \mathrm{proj}_{\succeq 0}(\bar{\mathbf{C}} - \mathcal{A}^*\bm{y}^*)$.

\subsection{Convex Optimization via CVXPY}
\label{subsec:cvxpy_solver}

The second approach utilizes CVXPY \cite{diamond2016cvxpy}, a high-level,
object-oriented Python modeling framework for convex optimization. CVXPY
automates problem canonicalization by transforming the abstract formulation
\eqref{eq:optimization_problem} into a standard conic form.

In our experiments, we employ CVXOPT \cite{andersen2011interior} as the backend
solver, which executes a primal-dual interior-point method. At each iteration
it solves a Newton system for a search direction along the central path of a
logarithmic-barrier reformulation of \eqref{eq:optimization_problem}. 
We refer to \cite{andersen2011interior} for the full algorithmic details; the
key point for our purposes is that each iteration requires forming and 
factorizing a dense $m \times m$ Kronecker-structured system, which is the 
source of the memory bottleneck discussed in Section~\ref{subsec:benchmarks}.

\subsection{Quadratic Semidefinite Programming via QSDPNAL}
\label{subsec:qsdpnal_solver}

The third approach implements QSDPNAL \cite{li2018qsdpnal}, a specialized MATLAB
package designed for high-dimensional quadratic semidefinite programming (QSDP).

\subsubsection{Problem Configuration Without Supplementary Constraints}
Without additional mean restrictions, the objective maps directly to the standard
framework handled by the \texttt{LSnal} routine
\begin{mini}|l|
{\mathbf{X}}{\frac{1}{2}\|\mathcal{H}(\mathbf{X})\|_F^2
+ \langle \mathbf{C}_0, \mathbf{X} \rangle}
{\label{eq:qsdp}}{}
\addConstraint{\mathcal{A}(\mathbf{X})}{= \bm{\alpha}}
\addConstraint{\mathbf{X}}{\succeq 0,}
\end{mini}
where $\mathcal{H}$ is set to the identity operator $\mathrm{Id}$, the constant
matrix is defined as $\mathbf{C}_0 = -\bar{\mathbf{C}}$, and $\mathbf{X}$ is
identified with the optimization variable $\mathbf{C}$.

The solver uses a two-level augmented Lagrangian method. The outer level
targets the linear equality constraints by solving a subproblem of the form:
\begin{equation}
\min_{\mathbf{X} \succeq 0} \;
\frac{1}{2}\|\mathcal{H}(\mathbf{X})\|_F^2
+ \langle \mathbf{C}_0 - \mathcal{A}^*(\bm{y}^k), \mathbf{X} \rangle
+ \frac{\sigma_k}{2}\|\mathcal{A}(\mathbf{X}) - \bm{\alpha}\|^2,
\label{eq:alm_subproblem}
\end{equation}
where $\bm{y}^k$ is the current estimate of the multiplier vector, and
$\sigma_k > 0$ is a penalty parameter dynamically adjusted across iterations.
The inner level handles the non-smooth subproblem utilizing a semi-smooth Newton~\cite{zhao2010newton} algorithm.

\subsubsection{Configuration with Additional Linear Inequality Constraints}
When an explicit mean constraint is added to bound or shift the generated entry
distributions, the problem is handled by the \texttt{LSnalI} function, which
expands \eqref{eq:qsdp} to incorporate linear inequality structures
\begin{mini}|l|
{\mathbf{X}}{\frac{1}{2}\|\mathcal{H}(\mathbf{X})\|_F^2
+ \langle \mathbf{C}_0, \mathbf{X} \rangle}
{\label{eq:qsdp_ineq}}{}
\addConstraint{\mathcal{A}(\mathbf{X})}{= \bm{\alpha}}
\addConstraint{\mathcal{B}(\mathbf{X})}{\geq \bm{\beta}}
\addConstraint{\mathbf{X}}{\succeq 0,}
\end{mini}
where $\mathcal{A}$ and $\bm{\alpha}$ are unchanged from \eqref{eq:qsdp}, and
the additional linear operator $\mathcal{B}$
and vector $\bm{\beta}$ encode supplementary linear
inequality constraints on $\mathbf{X}$.

\subsection{Numerical Post-Processing}
\label{subsec:postprocessing}

Due to finite numerical machine precision, solvers can output a reconstructed
matrix $\tilde{\mathbf{C}}$ whose minimum eigenvalue is slightly negative
(e.g., $\approx -10^{-9}$). To ensure that the output is strictly valid, we
introduce a uniform post-processing routine (for matrices with minimum eigenvalue which are slightly negative)
\begin{equation}
    \epsilon = \max\left(0, -\lambda_{\min}(\tilde{\mathbf{C}})\right).
\end{equation}
We then construct the diagonally shifted matrix
$\tilde{\mathbf{C}}_\epsilon = \tilde{\mathbf{C}} + \epsilon \mathbf{I}$, and
scale it to ensure a unit diagonal:
\begin{equation}
    \mathbf{C} = \frac{1}{1+\epsilon}\,\tilde{\mathbf{C}}_\epsilon.
    \label{eq:renorm}
\end{equation}
This operation yields a mathematically rigorous correlation matrix
$\mathbf{C} \in \mathcal{C}(G)$. In practice, $\epsilon$ remains bounded near
the machine precision limit, leaving the final projection error negligible.

\subsection{Comparison of the Three Methods}

Each of the three methods has distinct characteristics. The \textbf{Dual approach} is theoretically elegant and computationally efficient. It transforms the problem into an unconstrained optimization in the dual space, which can be solved using efficient quasi-Newton methods. \textbf{CVXPY} offers the greatest ease of use and flexibility. It is ideal for rapid prototyping and for incorporating additional constraints. However, it may not be the most efficient choice for very large-scale problems. \textbf{QSDPNAL} is specifically designed for quadratic PSD problems and can handle problems that are intractable for general-purpose solvers. However, it requires MATLAB and is less flexible in terms of adding arbitrary constraints compared to CVXPY.

In terms of memory usage, the limitation we observed in practice is specific to
CVXPY's interior-point backend (CVXOPT). Indeed, it explicitly forms and factorizes the dense Kronecker-structured KKT
system arising at each interior-point iteration (see
Section~\ref{subsec:cvxpy_solver}), its memory footprint scales poorly. This is
in contrast to the dual approach, which only stores a vector of size
$m = p+q$, and QSDPNAL, whose augmented Lagrangian and semismooth Newton scheme
avoid ever forming such a dense system explicitly.

\section{Numerical Experiments (Baseline)}
\label{sec:baseline_experiments}
The complete implementation of our algorithm is publicly accessible. The source codes are available at \url{https://gricad-gitlab.univ-grenoble-alpes.fr/fakhara/Graph-Based_Correlation_Matrix_Generation}.

\subsection{The Elliptope and Non-Uniformity}
\label{subsec:elliptope}

Several classical methods for generating correlation matrices are designed to
produce samples that are \emph{uniform} over the space of correlation matrices
$\mathcal{C}$. 

For instance, Joe~\cite{joe2006generating} characterizes the uniform
distribution over the elliptope $\mathcal{C}$ using partial correlations, and Lewandowski
et al.\ \cite{lewandowski2009generating} propose the vines and onion methods,
which also target uniformity. Designing an uniform sampling over the subset $\mathcal{C}(G)$ associated with a given graph
$G$ has only been achieved in the chordal graph setting, via a
Metropolis--Hastings-based approach~\cite{cordoba2020generating} building on
the Cholesky-based construction discussed in
Section~\ref{sec:state_of_the_art}.

Our proposed method does \emph{not} guarantee uniformity over $\mathcal{C}(G)$.
Rather, it produces a deterministic solution---the projection of $\bar{\mathbf{C}}$
onto $\mathcal{C}(G)$ in the Frobenius sense. We illustrate this non-uniformity
concretely in the simplest non-trivial case, $p = 3$ with the zero constraint
$c_{23} = 0$, i.e., the graph $G$ consists of edges $(1,2)$ and $(1,3)$ only.

To understand the geometric characteristics of the feasible space $\mathcal{C}(G)$,
we consider the complete graph with $p = 3$ nodes. A $3 \times 3$ correlation
matrix is defined by three off-diagonal scalar variables,
$x = c_{12}$, $y = c_{13}$, and $z = c_{23}$. The positive semidefiniteness
condition $\mathbf{C} \succeq 0$ implies the determinant must be non-negative:
\begin{equation}
    \det(\mathbf{C}) = 1 - x^2 - y^2 - z^2 + 2xyz \geq 0.
    \label{eq:elliptope_equation}
\end{equation}
The 3D volume bounded by \eqref{eq:elliptope_equation} is known as the
\emph{elliptope} $\mathcal{E}_3$.

A valid correlation matrix under the constraint $c_{23} = 0$ takes the form
\[
\mathbf{C} =
\begin{pmatrix}
1 & x' & y' \\
x' & 1 & 0 \\
y' & 0 & 1
\end{pmatrix},
\]
which is PSD if and only if $(x')^2 + (y')^2 \leq 1$, so the feasible set
$\mathcal{C}(G)$ is parameterized by the unit disk
$\mathcal{D} = \{(x', y') : (x')^2 + (y')^2 \leq 1\}$.

\subsubsection[Closed Form of the Projection in the p=3 Case]{Closed Form of the Projection in the $p=3$ Case}

Expanding the Frobenius objective $\frac{1}{2}\|\mathbf{C} - \bar{\mathbf{C}}\|_F^2$
entry by entry, and noting that the diagonal entries are fixed ($c_{ii} = 1$) and
$c_{23}$ is constrained to zero, only the free entries $c_{12} = x'$ and
$c_{13} = y'$ appear in the objective.
Specifically, problem~\eqref{eq:optimization_problem} reduces to:
\begin{mini}
{x',\, y'}{(x - x')^2 + (y - y')^2 + z^2}
{}{}
\addConstraint{(x')^2 + (y')^2 \leq 1,}
\label{eq:reduced_p3}
\end{mini}
where $z^2$ is a constant term arising from fixing $c_{23} = 0$. The closed-form
solution is:
\begin{equation}
    (\hat{c}_{12},\, \hat{c}_{13}) \;=\;
    \begin{cases}
        (x,\; y)
            & \text{if } x^2 + y^2 \leq 1, \\[6pt]
        \displaystyle\frac{(x,\; y)}{\sqrt{x^2 + y^2}}
            & \text{if } x^2 + y^2 > 1.
    \end{cases}
    \label{eq:proj_formula}
\end{equation}

\begin{figure}[!ht]
    \centering
    \begin{subfigure}[b]{0.42\textwidth}
        \centering
        \includegraphics[width=0.9\textwidth]{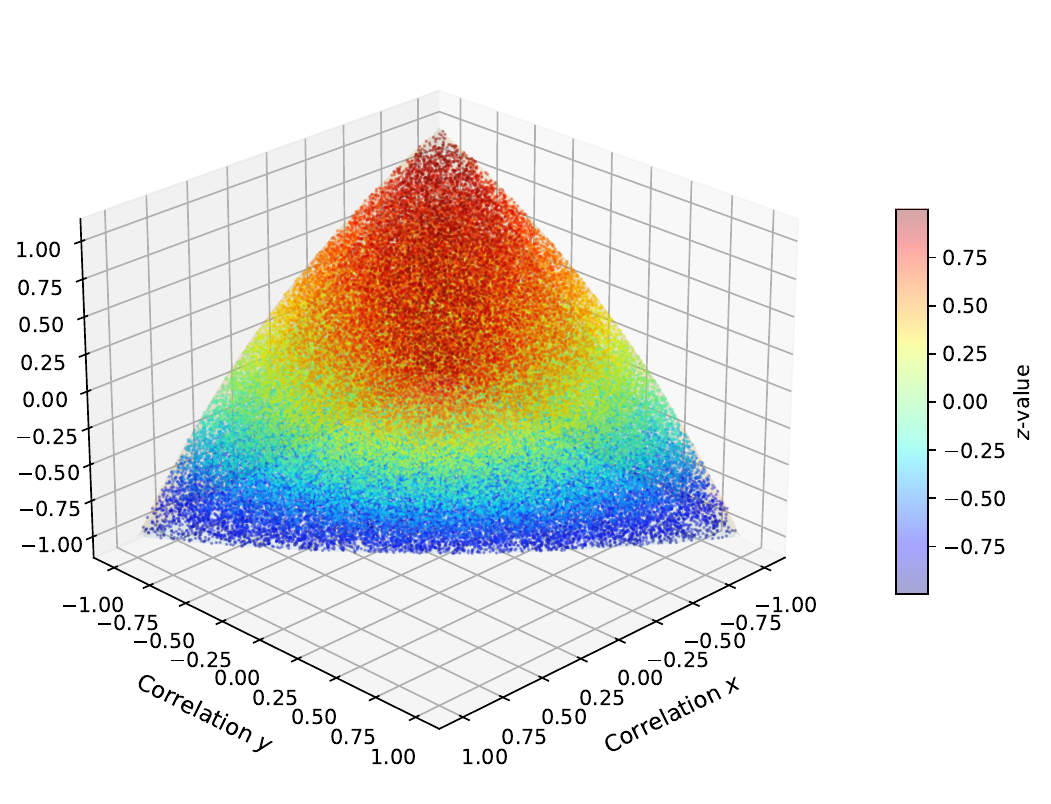}
        \caption{Three-dimensional visualization of the elliptope $\mathcal{E}_3$,
        the set of all $3 \times 3$ correlation matrices. The bounding surfaces
        represent the manifold where $\det(\mathbf{C}) = 0$. Points inside the
        volume (colored by $z$-value) are sampled uniformly from the interior via
        rejection sampling.}
        \label{fig:elliptope3d}
    \end{subfigure}
    \hspace{1cm}
    \begin{subfigure}[b]{0.42\textwidth}
        \centering
        \includegraphics[width=0.9\textwidth]{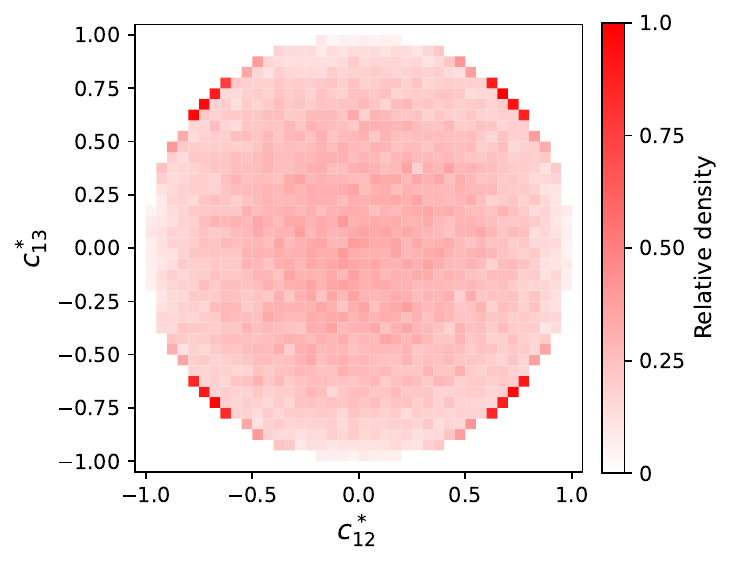}
        \caption{Density of the projected entries $(\hat c_{12}, \hat c_{13})$ over
        $N=100{,}000$ independent uniform draws of $\bar{\mathbf{C}}$ from
        $\mathcal{E}_3$, for the graph $G$ with edges $(1,2)$ and $(1,3)$ only
        (so $\hat c_{23} = 0$).}
        \label{fig:projection2d}
    \end{subfigure}
    \caption{Geometric illustration of the $p=3$ case. \emph{(a)}: the 3D
    elliptope $\mathcal{E}_3$. \emph{(b)}: the non-uniform density of projected
    solutions $(\hat c_{12}, \hat c_{13})$ over the unit disk $\mathcal{D}$, proving
    that the projection map does not preserve uniformity.}
    \label{fig:elliptope}
\end{figure}

To generate Figure~\ref{fig:elliptope}, we sample entries $x, y, z$ uniformly 
from $[-1, 1]$ and filter them via rejection sampling to satisfy the PSD 
condition ($\det(\mathbf{C}) > 0$), yielding a uniform distribution within 
the elliptope. Classical approaches such as the vines and onion procedures 
\cite{lewandowski2009generating} provide alternative mechanisms for generating 
samples uniformly over the elliptope without rejection sampling. We then project 
these points onto the subspace defined by the zero-constraint $c_{23} = 0$ using 
the closed-form formula \eqref{eq:proj_formula}.

\subsubsection{Non-uniformity of the projection}

For the $p=3$ case with the graph constraint $c_{23} = 0$, the constrained 
feasible set is
\begin{equation}
\mathcal{C}(G) = \left\{ \mathbf{C} = \begin{pmatrix} 1 & c_{12} & c_{13} \\ 
c_{12} & 1 & 0 \\ c_{13} & 0 & 1 \end{pmatrix} : \mathbf{C} \succeq 0 \right\},
\label{eq:elliptope_constrained}
\end{equation}
which is parameterized by the unit disk $\mathcal{D} = \{(c_{12}, c_{13}) : 
c_{12}^2 + c_{13}^2 \leq 1\}$.

The key observation is that \emph{uniform sampling from the full elliptope $\mathcal{E}_3$ 
does not induce a uniform distribution on the constrained subspace $\mathcal{C}(G)$}. 
If we sample $(x, y, z)$ uniformly from the 3D elliptope 
and then project onto the subspace $z = 0$ (or equivalently, apply the projection 
formula $\eqref{eq:proj_formula}$), the resulting distribution of $(c_{12}, c_{13})$ 
exhibits marked non-uniformity. 

The heatmap in Figure~\ref{fig:elliptope}(b) provides direct empirical evidence. Over $N=100{,}000$ uniform samples from $\mathcal{E}_3$, the projected entries 
$(c_{12}, c_{13})$ concentrate heavily at the boundary of the disk (dark red regions), 
particularly at the corners where $|c_{12}| \approx |c_{13}|$. The interior 
of the disk (light pink regions) is much less densely populated. This non-uniformity 
arises because the elliptope geometry in 3D compresses more measure toward the 
$(c_{12}, c_{13})$-boundary than toward the center. The projection operation 
further concentrates this mass at the boundary by mapping all exterior points 
radially inward.

Formally, the marginal density of $(c_{12}, c_{13})$ under the 
uniform distribution over $\mathcal{E}_3$ is obtained by 
integrating out $z$:
\begin{equation}
p(c_{12}, c_{13}) = \int_{z_{\min}(c_{12}, c_{13})}^{z_{\max}(c_{12}, c_{13})} 
f_{\mathcal{E}_3}(c_{12}, c_{13}, z) \, dz,
\end{equation}
where $f_{\mathcal{E}_3}$ denotes the uniform density over 
$\mathcal{E}_3$ and $z_{\min}, z_{\max}$ are the extremes of 
the feasible interval for $z$ given $(c_{12}, c_{13})$, 
determined by the PSD constraint 
$1 - c_{12}^2 - c_{13}^2 - z^2 + 2c_{12}c_{13}z \geq 0$. Two competing effects explain the non-uniform density visible in 
Figure~\ref{fig:elliptope}(b). First, at the \emph{boundary} of the disk 
$\partial\mathcal{D}$ (where $c_{12}^2 + c_{13}^2 = 1$), the feasible interval 
$[z_{\min}, z_{\max}]$ collapses to a single point, yet the projection 
\eqref{eq:proj_formula} maps \emph{all} exterior points 
$(x,y)$ with $x^2 + y^2 > 1$ radially onto $\partial\mathcal{D}$, 
concentrating a large mass there. Second, at the \emph{center} of the disk 
(where $c_{12}^2 + c_{13}^2 \approx 0$), points already lying inside the 
disk are kept fixed by \eqref{eq:proj_formula}, and the interval 
$[z_{\min}, z_{\max}]$ is at its widest, meaning more 
triples $(x, y, z)$ from the full elliptope project to nearby interior points, 
also raising the density there. At intermediate radii the two effects are weakest, 
yielding a lower density. The result is a non-uniform distribution over $\mathcal{C}(G)$, 
confirming that the projection map does not preserve uniformity.

This non-uniformity is not a deficiency of our method per se, but a consequence
of its design objective. Our approach generates correlation matrices as close as
possible to a given $\bar{\mathbf{C}}$ while respecting the prescribed graph
structure, making it particularly suitable for simulation studies where the
synthetic matrix should resemble a given empirical one.

\subsection{Effect of Graph Structure}
\label{subsec:graph_structure}

The graph structure $G$ is one of the main parameters governing the generated
correlation matrices. We consider five families of random graphs, which differ
both in their local configuration and in their global connectivity patterns.
Figure~\ref{fig:graph_examples} illustrates one small representative instance of
each family.

\begin{figure}[!ht]
    \centering
    \begin{minipage}{0.18\textwidth}
        \centering
        \begin{tikzpicture}[nodes={circle, draw, fill=blue!20, inner sep=2pt}, scale=0.85]
            \node (1) at (0,1.5) {\small 1};
            \node (2) at (1,1) {\small 2};
            \node (3) at (0.8,-0.5) {\small 3};
            \node (4) at (-0.8,-0.5) {\small 4};
            \node (5) at (-1,1) {\small 5};
            \draw (1)--(2) (1)--(4) (2)--(3) (2)--(5) (3)--(5) (4)--(5);
        \end{tikzpicture}
        \subcaption{Erd\H{o}s--R\'{e}nyi}
        \label{subfig:erdos}
    \end{minipage}
    \hfill
    \begin{minipage}{0.18\textwidth}
        \centering
        \begin{tikzpicture}[nodes={circle, draw, fill=red!20, inner sep=2pt}, scale=0.85]
            \foreach \i in {1,...,6}{
                \node (\i) at ({60*\i}:1.2cm) {\small \i};
            }
            \draw (1)--(2)--(3)--(4)--(5)--(6)--(1);
            \draw (1)--(4) (2)--(5);
        \end{tikzpicture}
        \subcaption{Watts--Strogatz}
        \label{subfig:watts}
    \end{minipage}
    \hfill
    \begin{minipage}{0.18\textwidth}
        \centering
        \begin{tikzpicture}[nodes={circle, draw, fill=green!20, inner sep=2pt}, scale=0.85]
            \node (H) at (0,0) {\small H};
            \foreach \i in {1,...,5}{
                \node (\i) at ({72*\i}:1.2cm) {\small \i};
                \draw (H)--(\i);
            }
            \draw (1)--(2);
        \end{tikzpicture}
        \subcaption{Barab\'{a}si--Albert}
        \label{subfig:barabasi}
    \end{minipage}
    \hfill
    \begin{minipage}{0.18\textwidth}
        \centering
        \begin{tikzpicture}[nodes={circle, draw, fill=orange!20, inner sep=2pt}, scale=0.85]
            \node (1) at (0,1) {\small 1};
            \node (2) at (1,0) {\small 2};
            \node (3) at (0,-1) {\small 3};
            \node (4) at (-1,0) {\small 4};
            \draw (1)--(2)--(3)--(4)--(1);
            \draw (1)--(3);
        \end{tikzpicture}
        \subcaption{Chordal}
        \label{subfig:chordal}
    \end{minipage}
    \hfill
    \begin{minipage}{0.18\textwidth}
        \centering
        \begin{tikzpicture}[nodes={circle, draw, fill=purple!20, inner sep=2pt}, scale=0.85]
            \node (1) at (0,1.2) {\small 1};
            \node (2) at (1,0.4) {\small 2};
            \node (3) at (0.7,-1) {\small 3};
            \node (4) at (-0.7,-1) {\small 4};
            \node (5) at (-1,0.4) {\small 5};
            \draw (1)--(2) (3)--(4) (4)--(5) (1)--(5);
            \draw[dashed] (2)--(4);
        \end{tikzpicture}
        \subcaption{SBM}
        \label{subfig:sbm}
    \end{minipage}
    \caption{Illustrative examples of the five graph families considered in our experiments. From left to right: Erd\H{o}s--R\'{e}nyi \cite{erdos1960evolution}, Watts--Strogatz \cite{watts1998collective}, Barab\'{a}si--Albert \cite{albert1999diameter}, Chordal \cite{vandenberghe2015chordal}, and Stochastic Block Model (SBM) \cite{abbe2018community}.}
    \label{fig:graph_examples}
\end{figure}
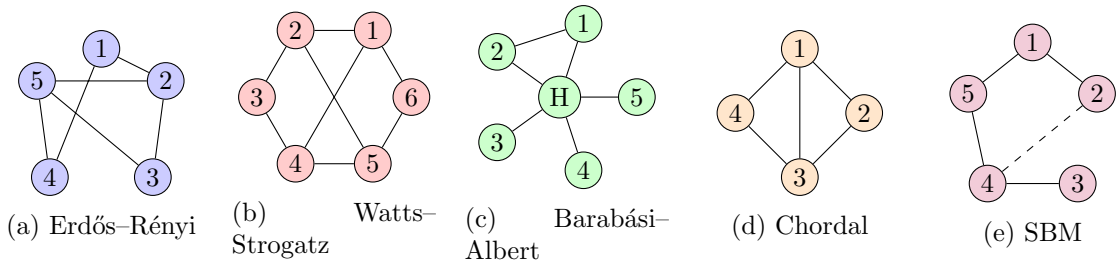

\paragraph{Erd\H{o}s--R\'{e}nyi graphs \cite{erdos1960evolution}}
In an Erd\H{o}s--R\'{e}nyi random graph $G(p, d)$, each of the $\binom{p}{2}$
possible edges between $p$ nodes is included independently with probability
$d \in [0,1]$, referred to as the \emph{edge density}. 
The degree distribution is binomial, and there are
no hubs or community structures. This makes it a natural baseline model for sparse
random zero patterns, and the density parameter $d$ directly controls the
proportion of non-zero entries in the generated correlation matrix. However, the resulting graph has no particular structure bias.

\paragraph{Watts--Strogatz graphs \cite{watts1998collective}}
The Watts--Strogatz model generates \emph{small-world} networks by starting from
a regular ring lattice and rewiring edges with a given probability. The resulting
graphs exhibit high clustering coefficients combined with short average path
lengths. In terms of the generated correlation matrices, this local structure
induces clusters of correlated variables separated by sparse inter-cluster
connections.

\paragraph{Barab\'{a}si--Albert graphs \cite{albert1999diameter}}
The Barab\'{a}si--Albert model generates \emph{scale-free} networks via a
preferential attachment mechanism. New nodes are more likely to connect to already
highly connected nodes (hubs, for example like node (H) in Figure~\ref{subfig:barabasi}). The resulting degree distribution follows a power
law. The presence of hubs means that a few variables are correlated with many
others, which tends to create denser rows and columns in the correlation matrix.

\paragraph{Stochastic Block Models \cite{abbe2018community}}
In a Stochastic Block Model (SBM), nodes are partitioned into blocks
(communities), with higher edge probability within blocks than between them. This
induces a block structure in the zero pattern of the correlation matrix, naturally
modeling scenarios where variables can be grouped into correlated clusters.

\paragraph{Chordal graphs \cite{vandenberghe2015chordal}}
A graph is chordal if every cycle of length four or more contains a chord (an
edge connecting two non-adjacent vertices in the cycle). Chordal graphs are of
particular importance in graphical models because they admit a \emph{perfect
elimination ordering} (PEO). In our study, chordal graphs are generated by
constructing random split graphs \cite{bender1985almost}, which form a well-known
subclass of chordal graphs. Among the ideas for generating chordal graphs randomly \cite{cseker2022generation}, we
use split graphs because they give us more direct control over the density of
the graph. Given a number of nodes $n$ and a target edge density
$d \in (0, 1)$, we calculate the exact target number of edges
$m = \lfloor d \frac{n(n-1)}{2} \rceil$. The graph vertices are partitioned into
a clique of size $k$ and an independent set of size $n-k$. A valid clique size
$k \in \{1, \dots, n\}$ satisfying the bounds
$\frac{k(k-1)}{2} \leq m \leq \frac{k(k-1)}{2} + k(n-k)$
is selected uniformly at random from the feasible set. After adding all internal
edges within the clique, we sample exactly $m - \frac{k(k-1)}{2}$ edges uniformly
at random from the available bipartite connections between the clique and the
independent set. By construction, the resulting graph is a split graph, ensuring
it is strictly chordal while precisely matching the prescribed edge density.
While many competing methods are restricted to chordal graphs, our
optimization-based approach handles all five families.

Figure~\ref{fig:fig3} displays the distribution of correlation values for various graph types at 50\% density, with each distribution averaged over 50 independent samples generated via CVXPY. As graph density increases, the choice of graph structure exerts no significant effect on the distribution of non-zero, non-diagonal entries, as all configurations converge toward the complete graph. Figure~\ref{fig:fig4} compares the computational execution times using CVXPY for a fixed dimension of $p=51$ across varying graph models and densities, with each data point averaged over 50 randomly generated graphs. While overall execution time decreases as density increases, empirical observations indicate that for a fixed density, computational performance varies across different graph types.

\begin{figure}[!ht]
\begin{multicols}{2}
    \includegraphics[width=0.9\linewidth]{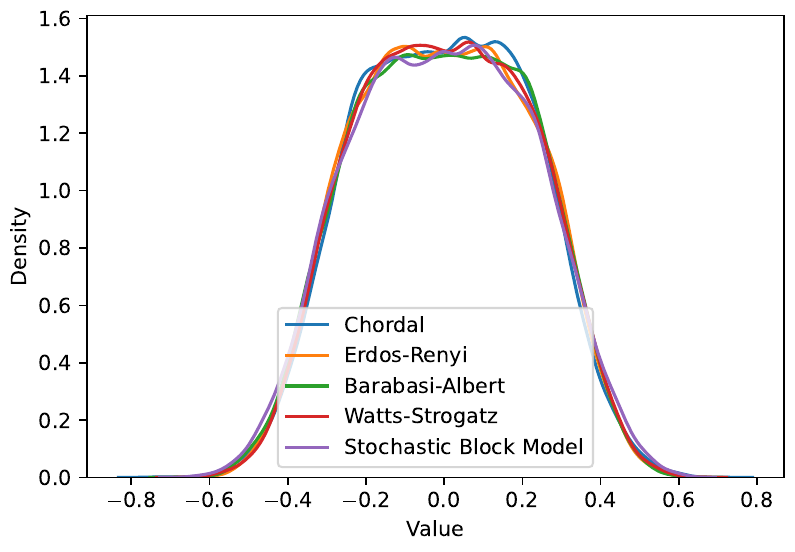}
\caption{Distribution of non-diagonal, non-zero correlation values for five graph
types at 50\% edge density, averaged over 50 i.i.d.\ randomly generated graphs
with $p = 51$ nodes, solved via CVXPY.}
\label{fig:fig3}
\columnbreak

\vspace*{-0.8cm}

    \includegraphics[width=0.9\linewidth]{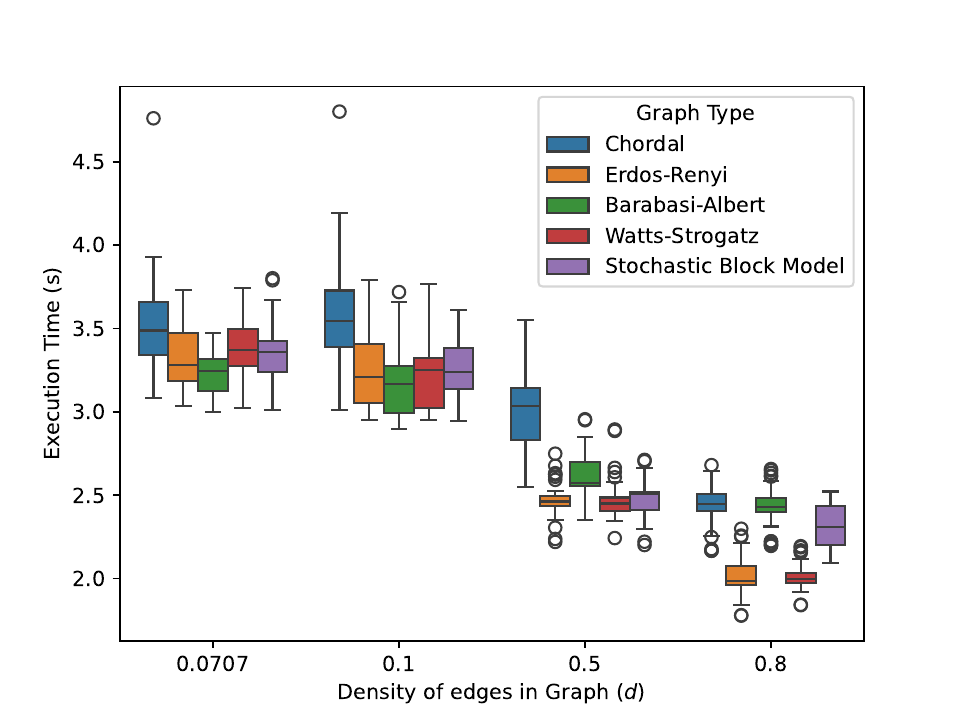}
    \caption{Comparison of execution times using CVXPY for a fixed dimension of
    51 across varying graph densities and models (averaged over 50 i.i.d.\ randomly generated graphs for each case.)}
    \label{fig:fig4}
  \end{multicols}  
\end{figure}

\subsection{Computational Cost and Scaling Benchmarks}
\label{subsec:benchmarks}

We evaluate the timing for dimensions $p \in \{51, 100, 1000\}$ on a random
Erd{\H{o}}s-R{\'e}nyi graph with edge density $d = 0.5$. All methods were run
with tolerance $10^{-6}$ and a maximum of $10{,}000$ iterations. The results are
summarized in Table~\ref{tab:computation_times}.

\begin{table}
    \centering
    \caption{Computation times (in seconds) 
    for the three methods
    across three problem dimensions $p$. All methods converged to the optimal
    solution. For $p = 1000$, CVXPY switched from CVXOPT to the SCS solver due
    to memory limitations.}
    \label{tab:computation_times}
    \begin{tabular}{lccc}
        \hline
        \textbf{Method} & $p = 51$ & $p = 100$ & $p = 1000$ \\
        \hline
        \textbf{Dual approach \cite{malick2004dual}} & 0.039 s & 0.090 s & 10.80 s \\
        \textbf{CVXPY (CVXOPT / SCS) \cite{diamond2016cvxpy}} & 2.76 s & 48.78 s & 9.49 h \\
        \textbf{QSDPNAL \cite{li2018qsdpnal}} & 0.50 s & 0.51 s & 48.39 s \\
        \hline
    \end{tabular}
\end{table}

The empirical benchmarks reveal clear architectural differences:
\begin{itemize}
    \item \textbf{CVXPY (CVXOPT):} Suffers from steep memory and computational
    scaling due to the construction and factorization of the massive
    Kronecker-structured KKT system described in
    Section~\ref{subsec:cvxpy_solver}. It becomes impractical for higher
    dimensions.
    \item \textbf{Dual Approach (L-BFGS-B):} Demonstrates robust scalability because it completely bypasses the need to
    store secondary barrier Hessians, operating directly on a dual variable space
    of size $m = p + q$.
    \item \textbf{QSDPNAL:} Outperforms CVXPY by multiple
    orders of magnitude as $p$ scales. By integrating a semi-smooth Newton method
    inside an augmented Lagrangian routine, it achieves remarkable speed and
    enables high-dimensional generation ($p = 1000$) in under 50 seconds.
\end{itemize}

All three methods achieve the same objective value across all dimensions (to within
numerical tolerance), confirming that they converge to the same unique solution of
problem~\eqref{eq:optimization_problem}. For moderate dimensions
($p \lesssim 100$), any of the three methods is applicable; for large-scale
problems ($p \gtrsim 500$), the dual approach or QSDPNAL should be preferred,
with CVXPY reserved for prototyping or when additional constraint flexibility is
required.

\subsection{Comparison with Baseline Methods}

Figure~\ref{fig:density_comparison} illustrates the empirical densities of
non-diagonal, non-zero entries produced by each method. All methods are evaluated
on 50 independently generated Erd\H{o}s-R\'enyi graphs with $p = 51$ nodes and
graph density $d = 0.5$. The diagonal dominance method concentrates entries near zero (see
\ref{app:bimodal} for a derivation), while partial orthogonalization produces
a wider distribution centered near zero, showing improvement over diagonal
dominance. Our method with $\bar{\mathbf{C}}$ which is realization of entry wise uniform distribution over the range of $[-1, 1 ]$ has a wider support, but still real world data motivating us with the additional mean constraint developed in the next section.

\begin{figure}[!ht]
    \centering
    \includegraphics[width=0.45\linewidth]{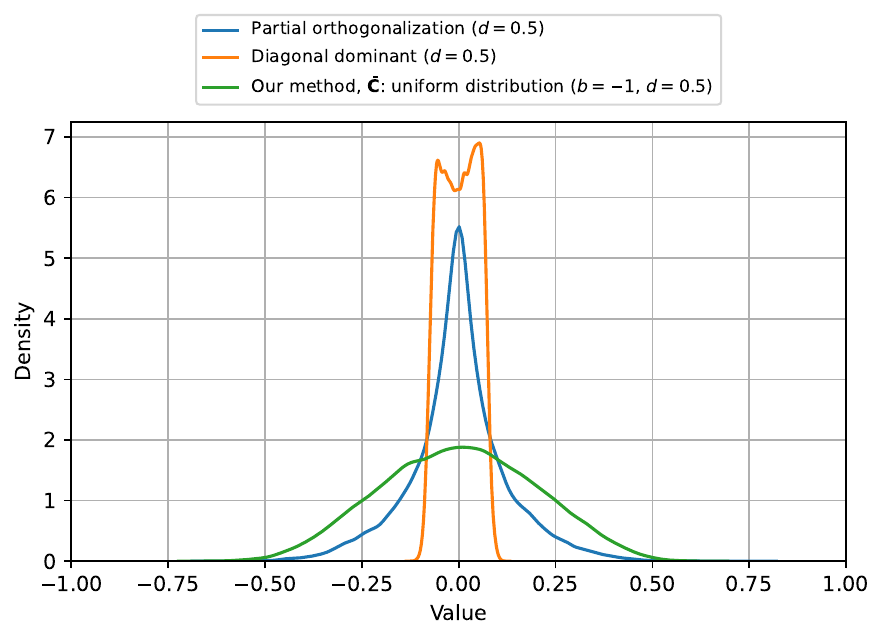}
\caption{Density of non-diagonal, non-zero entries in generated correlation
matrices for five methods (diagonal dominance, partial orthogonalization, our
method with uniform initialization, our method with empirical initialization, and
the empirical data itself), evaluated on 50 i.i.d.\ Erd\H{o}s--R\'enyi graphs
with $p = 51$ nodes and edge density $d = 0.5$.}
\label{fig:density_comparison}
\end{figure}

\section{Additional Constraint: Mean Threshold}
\label{sec:mean_constraint}

A significant limitation of existing correlation generation methods is their
tendency to produce entry distributions centered around zero. Empirical evidence
from diverse fields, such as functional brain connectivity (fMRI) in
neuroscience and asset returns in finance, frequently reveals a systematic
positive shift in the correlation structure. In this section we dive in more theoritical aspect of it and in next section we present numerical results.

To address this limitation, we extend our optimization framework
\eqref{eq:optimization_problem} by introducing a linear inequality constraint that
enforces a minimum threshold on the average value of the generated off-diagonal
entries.

\subsection{Incorporating the Mean Constraint}

To generate correlation matrices that can match the distributional characteristics observed in real data, we extend the basic optimization problem \eqref{eq:optimization_problem} by introducing a constraint on the mean of the non-diagonal entries. Specifically, for a given threshold parameter $t \geq -1$, we impose the constraint:
\begin{equation}
\frac{1}{2|E|} \sum_{\substack{i,j=1\\i \neq j\\(i,j) \in E}}^p c_{ij} \geq t,
\label{eq:mean_constraint}
\end{equation}
which ensures that the average correlation among variables connected in the graph is at 
least $t$. Here, $t \in [-1, 1]$ plays the role of a \emph{threshold} on the mean 
correlation, where larger values of $t$ enforce a stronger positive shift in the generated 
entries.

This constraint can be equivalently written in compact form as 
$\frac{1}{2|E|} \langle \mathbf{C} - \mathbf{I}, \mathbf{1}_{p \times p} \rangle \geq t$, 
where $\mathbf{1}_{p \times p}$ denotes the matrix of all ones. Setting $t = -1$ (or any 
value $t \leq -1$) is equivalent to imposing no constraint on the mean, as correlation 
coefficients are naturally bounded below by $-1$. The extended optimization problem becomes:
\begin{mini}[2]
{\mathbf{C}}{\frac{1}{2} \|\mathbf{C} - \bar{\mathbf{C}}\|_F^2}
{}{}
\addConstraint{\mathbf{C}}{\in \mathcal{C}(G)}
\addConstraint{\frac{1}{2|E|} \sum_{\substack{i,j=1\\i \neq j\\(i,j) \in E}}^p c_{ij}}{\geq t.}
\label{eq:optimization_problem_with_mean}
\end{mini}
The introduction of this linear inequality constraint preserves the convexity of the feasible set, as it defines a half-space in the space of symmetric matrices. Consequently, problem \eqref{eq:optimization_problem_with_mean} remains a convex quadratic program with a unique solution whenever the constraint set is non-empty. By varying the threshold parameter $t$, practitioners can control the shift in the distribution of correlation values, thereby generating synthetic data that matches specific empirical characteristics or studying the impact of different signal-to-noise ratios in simulation studies. Beyond controlling the mean, this framework can accommodate other constraints 
of practical interest, such as lower bounds on the minimum correlation (useful 
for controlling signal-to-noise ratios), upper bounds on the maximum correlation, 
or variance and kurtosis constraints on the off-diagonal entries, provided they 
maintain the convexity of the feasible region.

\subsection{Theoretical Bound on the Threshold Value}
\label{subsec:threshold_bounds}
Unlike problem~\eqref{eq:optimization_problem}, whose feasible set is always
non-empty (the identity matrix trivially satisfies $\mathcal{C}(G)$ for any
$G$), the mean-constrained problem~\eqref{eq:optimization_problem_with_mean}
need not be feasible. Indeed, for a given threshold $t$, the intersection of
$\mathcal{C}(G)$ with the half-space $\{\mathbf{C}:\text{mean}\geq t\}$ can be
empty, depending on both $t$ and the structure of $G$.

Before analyzing the structural limitations imposed by specific graph structures,
it is instructive to examine the simplest possible case, the \emph{complete graph}
$K_p$, where every pair of nodes is connected and no zero pattern is imposed.
A natural family of structured correlation matrices is the constant off-diagonal family
\begin{equation}
    \mathbf{C}_\tau = (1-\tau)\mathbf{I} + \tau \mathbf{1}\mathbf{1}^\top,
    \label{eq:equicorr}
\end{equation}
where $\tau \in \mathbb{R}$ and $\mathbf{1} \in \mathbb{R}^p$ is the all-ones
vector, so the mean off-diagonal entry equals $\tau$ exactly.

Since $\mathbf{C}_\tau$ is a rank-one perturbation
of the identity, its eigenvalues are $\lambda_1 = 1 + (p-1)\tau$ (eigenvector
$\mathbf{1}$) and $\lambda_2 = 1 - \tau$ (multiplicity $p-1$). The condition
$\mathbf{C}_\tau \succeq 0$ therefore requires
\begin{equation}
    -\frac{1}{p-1} \leq \tau \leq 1.
    \label{eq:equicorr_range}
\end{equation}
The upper bound $\tau \leq 1$ is dimension-independent, while the lower bound
$\tau \geq -1/(p-1)$ tightens as $p$ grows, making strongly negative global
correlations structurally impossible in high dimensions.

\subsubsection{Circle Graphs}
\label{subsubsec:threshold_bounds}

To formalize the limits imposed by the graph structure, consider a general
chordless cycle $C_k$ of length $k \geq 4$, with vertices $V = \{1, 2, \dots, k\}$
and edges $E = \{(1,2), (2,3), \dots, (k-1,k), (k,1)\}$. For any valid
correlation matrix $\mathbf{C} \in \mathcal{C}(C_k)$, the diagonal entries are
$c_{ii} = 1$, and non-edges must vanish ($c_{ij} = 0$ for all $(i,j) \notin E$).

A fundamental property of the PSD cone is that the Frobenius
inner product of any two PSD matrices is non-negative~\cite[Section 2.6.1]{boyd2004convex}. Therefore, for any PSD test
matrix $\mathbf{M} \succeq 0$, we must have
$\langle \mathbf{C}, \mathbf{M} \rangle = \operatorname{Tr}(\mathbf{C}\mathbf{M})
\geq 0$. Because $c_{ij} = 0$ for non-edges, the cross-terms vanish and the
inequality simplifies to
\begin{equation}
    \sum_{i=1}^{k} m_{ii}
    + 2 \sum_{(i,j) \in E} c_{ij}\, m_{ij} \geq 0.
    \label{eq:psd_inner_product}
\end{equation}

We construct $\mathbf{M}$ explicitly as a Gram matrix to guarantee
$\mathbf{M} \succeq 0$. Specifically, let $r \geq 1$ be an integer and let
$\bm{u}_1, \dots, \bm{u}_k \in \mathbb{R}^r$ be a collection of vectors, one
per vertex of the cycle. We define the $(i,j)$ entry of $\mathbf{M}$ by
\begin{equation}
    m_{ij} = \bm{u}_i^\top \bm{u}_j,
    \label{eq:gram_entry}
\end{equation}
so that $\mathbf{M}$ is the Gram matrix of $\{\bm{u}_i\}_{i=1}^k$, which is
PSD by construction. We further impose the normalization $\|\bm{u}_i\|^2 = 1$
for all $i$, so that $m_{ii} = \bm{u}_i^\top \bm{u}_i = 1$. The goal is to
choose the vectors $\bm{u}_1, \dots, \bm{u}_k$ such that all adjacent inner
products take a common negative value:
\begin{align}
    m_{i,i+1} &= \bm{u}_i^\top \bm{u}_{i+1} = -\alpha \quad \forall\, (i,i+1) \in E,\nonumber \\
    m_{k,1} &= \bm{u}_k^\top \bm{u}_1 = -\alpha, \label{eq:uniform_weight_b}
\end{align}
for some $\alpha > 0$. Under this uniform edge-weight construction, substituting
$m_{ii} = 1$ and $m_{ij} = -\alpha$ for all $(i,j) \in E$
into \eqref{eq:psd_inner_product} gives:
\begin{equation}
    k - 2\alpha \sum_{(i,j) \in E} c_{ij} \geq 0,
    \quad \text{i.e.,} \quad
    \frac{1}{k}\sum_{(i,j) \in E} c_{ij} \leq \frac{1}{2\alpha}.
    \label{eq:general_bound}
\end{equation}
The bound \eqref{eq:general_bound} becomes tighter as $\alpha$ increases.
The maximum achievable $\alpha$, subject to the existence of unit-norm vectors
$\bm{u}_1, \dots, \bm{u}_k \in \mathbb{R}^r$ (for some $r$) satisfying
\eqref{eq:uniform_weight_b}, depends on the parity of $k$.

\textbf{Case 1: Even chordless cycles} ($k \geq 4$, $k$ even).
For even $k$, the uniform negativity condition \eqref{eq:uniform_weight_b} can be
satisfied in $d = 1$ dimension. Define the scalar assignments
\begin{equation}
    u_j = (-1)^j \in \mathbb{R}, \quad j = 1, \dots, k,
    \label{eq:scalar_even}
\end{equation}
so that each $u_j \in \{-1, +1\}$ with $u_j^2 = 1$, satisfying the unit-norm
condition. Since $r=1$, the Gram matrix entry \eqref{eq:gram_entry} reduces to
the scalar product:
\begin{equation}
    m_{ij} = u_i \cdot u_j = (-1)^i \cdot (-1)^j = (-1)^{i+j}.
    \label{eq:gram_even}
\end{equation}
We verify \eqref{eq:uniform_weight_b} for all edges. For any interior adjacent
edge $(i, i+1)$:
\begin{equation*}
    m_{i,i+1} = (-1)^{i+(i+1)} = (-1)^{2i+1} = -1.
\end{equation*}
For the wrap-around edge $(k, 1)$:
\begin{equation*}
    m_{k,1} = (-1)^{k+1} = -1,
\end{equation*}
where the last equality holds because $k$ is even, making $k+1$ odd.
Thus condition \eqref{eq:uniform_weight_b} is satisfied with $\alpha = 1$, and
substituting into \eqref{eq:general_bound} yields:
\begin{equation}
    \frac{1}{k} \sum_{(i,j) \in E} c_{ij} \leq \frac{1}{2}.
    \label{eq:bound_even}
\end{equation}

\textbf{Case 2: Odd chordless cycles} ($k \geq 5$, $k$ odd).
For odd $k$, no scalar assignment ($r = 1$) can satisfy the uniform negativity
condition \eqref{eq:uniform_weight_b}. We therefore lift to
$r = 2$ dimensions. Define the unit vectors
\begin{equation}
    \bm{u}_j =
    \begin{pmatrix} \cos(j\theta) \\ \sin(j\theta) \end{pmatrix}
    \in \mathbb{R}^2,
    \qquad \theta = \frac{k-1}{k}\,\pi,
    \label{eq:trig_vectors}
\end{equation}
so that $\|\bm{u}_j\|^2 = \cos^2(j\theta) + \sin^2(j\theta) = 1$. The Gram
matrix entry \eqref{eq:gram_entry} evaluates via the cosine substraction formula:
\begin{equation}
    m_{ij} = \bm{u}_i^\top \bm{u}_j
    = \cos\bigl((i - j)\theta\bigr).
    \label{eq:gram_odd}
\end{equation}
We verify \eqref{eq:uniform_weight_b} for all edges. For any interior adjacent
edge $(i, i+1)$:
\begin{equation*}
    m_{i,i+1} = \cos\bigl((i-(i+1))\theta\bigr)
    = \cos(-\theta) = -\cos\!\left(\frac{\pi}{k}\right).
\end{equation*}
For the wrap-around edge $(k, 1)$:
$$m_{k,1} = \cos\bigl((k-1)\theta\bigr) = \cos\left(\frac{(k-1)^2}{k}\pi\right) = \cos\left(\left(k - 2 + \frac{1}{k}\right)\pi\right).$$
Since $k$ is odd, $k - 2$ is also odd, so $\cos((k-2)\pi) = -1$.
Using the addition formula:
\begin{equation*}
    m_{k,1}
    = \cos\!\left((k-2)\pi + \frac{\pi}{k}\right)
    = -\cos\!\left(\frac{\pi}{k}\right),
\end{equation*}
where we used $\sin((k-2)\pi) = 0$. Thus condition \eqref{eq:uniform_weight_b}
is satisfied with $\alpha = \cos(\pi/k)$ for all edges, and substituting into
\eqref{eq:general_bound} yields:
\begin{equation}
    \frac{1}{k} \sum_{(i,j) \in E} c_{ij} \leq \frac{1}{2\cos(\pi/k)}.
    \label{eq:bound_odd}
\end{equation}

The two cases are unified by noting that as $k \to \infty$,
$\cos(\pi/k) \to 1$ and both bounds converge to $1/2$. For odd $k$, the bound
$1/(2\cos(\pi/k))$ is strictly greater than $1/2$ and starts at
$1/(2\cos(\pi/5)) \approx 0.618$ for $C_5$, decreasing monotonically to $0.5$.

In either case, these bounds show that the mean correlation along the edges
of a chordless cycle $C_k$ cannot be pushed arbitrarily close to $1$.

\subsubsection{Tridiagonal (Path) Graphs}
\label{subsubsec:tridiagonal_bound}

A complementary structural bound arises for the \emph{path graph} $P_k$, whose
associated correlation matrices are exactly the tridiagonal matrices
$\mathbf{C} \in \mathcal{C}(P_k)$: vertices $V = \{1, \dots, k\}$ and edges
$E = \{(1,2), (2,3), \dots, (k-1,k)\}$, so that
\begin{equation}
\mathbf{C} =
\begin{pmatrix}
1     & b_1    &        &        &  \\
b_1   & 1      & b_2    &        &  \\
      & b_2    & 1      & \ddots &  \\
      &        & \ddots & \ddots & b_{k-1} \\
      &        &        & b_{k-1} & 1
\end{pmatrix},
\label{eq:tridiag_matrix}
\end{equation}
with $b_i \coloneqq c_{i,i+1}$ for $i = 1, \dots, k-1$ denoting the off-diagonal
entries along the path, and all entries outside the tridiagonal band forced to
zero by the graph structure. Unlike the cycle case, which we treated via a
Gram-matrix PSD argument, the tridiagonal structure admits a sharper tool from
the theory of chain sequences and continued fractions
\cite{andjelic2011sufficient}.

\begin{definition}[Chain sequence~\cite{andjelic2011sufficient}]
A sequence $\{\alpha_i\}_{i=1}^{k-1}$ with $\alpha_i > 0$ is a (finite, positive)
\emph{chain sequence} if there exists a parameter sequence $\{g_i\}_{i=0}^{k-1}$
with $0 \le g_0 < 1$ and $0 < g_i < 1$ for $i \ge 1$, such that
$\alpha_i = g_i(1 - g_{i-1})$ for all $i$.
\end{definition}

The positive definiteness of a general real symmetric tridiagonal matrix
\begin{equation}
\mathbf{A}_k =
\begin{pmatrix}
a_1   & b_1    &        &        &  \\
b_1   & a_2    & b_2    &       &  \\
      & b_2    & a_3    & \ddots &  \\
    &       & \ddots & \ddots & b_{k-1} \\
      &        &        & b_{k-1} & a_k
\end{pmatrix}, \qquad a_i > 0,
\label{eq:tridiag_general}
\end{equation}
is fully characterized in terms of chain sequences by the following classical
result.

\begin{theorem}[Wall--Wetzel~\cite{andjelic2011sufficient}]
\label{thm:wall_wetzel}
The real symmetric tridiagonal matrix $\mathbf{A}_k$ in~\eqref{eq:tridiag_general},
with diagonal entries $a_1, \dots, a_k > 0$, is positive definite if and only if
$\left\{ \dfrac{b_i^2}{a_i a_{i+1}} \right\}_{i=1}^{k-1}$ is a chain sequence.
\end{theorem}

For our tridiagonal correlation matrix~\eqref{eq:tridiag_matrix}, $a_i \equiv 1$,
so Theorem~\ref{thm:wall_wetzel} reduces to $\mathbf{C} \succ 0$ if and only if
$\{b_i^2\}_{i=1}^{k-1}$ is a chain sequence. The following result gives the
exact threshold for the largest \emph{constant} chain sequence.

\begin{theorem}[Ismail and Li~\cite{ismail1992bound}]
\label{thm:constant_chain}
The constant sequence $\{\alpha\}_{i=1}^{k-1}$ is a chain sequence if and only if
\begin{equation}
    0 < \alpha < \frac{1}{4\cos^2\left(\dfrac{\pi}{k+1}\right)}.
    \label{eq:chain_threshold}
\end{equation}
\end{theorem}

\textit{Application to the symmetric tridiagonal family.} Mirroring the
equicorrelation construction $\mathbf{C}_\alpha$ used for the complete graph
in~\eqref{eq:equicorr}, consider the symmetric tridiagonal family with
constant off-diagonal weight $b_i \equiv b$ for all $i = 1, \dots, k-1$ in
\eqref{eq:tridiag_matrix}. By Theorem~\ref{thm:wall_wetzel}, this matrix is
positive definite if and only if $\{b^2\}_{i=1}^{k-1}$ is a chain sequence,
which by Theorem~\ref{thm:constant_chain} holds if and only if
$|b| < 1/(2\cos(\pi/(k+1)))$. Since the mean off-diagonal entry of this family is
exactly $\frac{1}{k-1}\sum_{i=1}^{k-1} b_i = b$, this yields, within the
constant-entry ansatz, the bound
\begin{equation}
    \frac{1}{k-1} \sum_{i=1}^{k-1} b_i < \frac{1}{2\cos\left(\dfrac{\pi}{k+1}\right)}
    \label{eq:tridiag_bound}
\end{equation}
on the mean correlation along the edges of a path graph $P_k$, in direct
analogy with the cycle bounds~\eqref{eq:bound_even}--\eqref{eq:bound_odd}. As
$k \to \infty$, the right-hand side of \eqref{eq:tridiag_bound} converges to
$1/2$, matching the limiting value of the even-cycle bound
\eqref{eq:bound_even}, while for small $k$ the bound is strictly looser than
$1/2$ (e.g., for $k=3$, the right-hand side equals
$1/(2\cos(\pi/4)) = 1/\sqrt{2} \approx 0.707$).

These bounds are not strict, they are attained at the boundary of the feasible
region, which makes the relationship between graph structure and the
feasibility of the mean constraint~\eqref{eq:mean_constraint} transparent. The
feasible region of problem~\eqref{eq:optimization_problem_with_mean} is the
intersection of the PSD cone, the unit-diagonal affine subspace, the graph
zero constraints, and the half-space defined by the mean threshold $t$. Any
value of $t$ exceeding a structurally-imposed upper bound renders this
intersection empty, and any optimization algorithm attempting to enforce
\eqref{eq:mean_constraint} beyond such a bound is mathematically guaranteed to
fail.

It is tempting to conclude, on the basis of the cycle bounds
\eqref{eq:bound_even}--\eqref{eq:bound_odd}, that chordality alone is
sufficient to avoid such bottlenecks, since chordal graphs are by definition
free of chordless cycles of length $4$ or greater. However,
Section~\ref{subsubsec:tridiagonal_bound} shows this is not the case. Indeed, the
path graph $P_k$ is chordal (indeed acyclic), yet for $k \geq 4$ it still admits a strict upper bound $1/(2\cos(\pi/(k+1))) < 1$ on its mean correlation,
arising not from a chordless cycle but from the chain-sequence constraint of
Theorem~\ref{thm:wall_wetzel}. 

\section{Experiments and Real-World Applications}
\label{sec:applications}

In this section, we present numerical and real-world experiments that complement
the theoretical analysis of the preceding sections. We begin by empirically
mapping the feasibility region of problem~\eqref{eq:optimization_problem_with_mean}
as a joint function of the edge density $d$ and the mean threshold $t$, across
four graph families. We then compare our framework against the GAN-based approach
of \cite{Gautier2020} and validate it on two real-world datasets, resting-state
fMRI recordings from neuroscience and stock return correlations from financial
markets.

\subsection{Feasibility Regions and Bottlenecks}
\label{subsec:feasibility}

The feasibility heatmaps in Figure~\ref{fig:feasibility_heatmaps} display, for
each graph model, the proportion of instances for which a feasible solution to
problem~\eqref{eq:optimization_problem_with_mean} exists, as a joint function of
the edge density $d$ (x-axis) and the mean threshold $t$ (y-axis). Each pixel is
colored according to the proportion of trials at that $(d,t)$ configuration for
which a feasible solution was found. Fully red pixels correspond to proportion 1
(always feasible), while lighter pixels indicate partial or complete infeasibility.
A key observation is that feasibility is highly sensitive to both $d$ and $t$, and
the shape of the feasible regions is directly explained by the cycle structure of
the graph, as established in Section~\ref{subsec:threshold_bounds}.

The choice of initial matrix $\bar{\mathbf{C}}$ has no effect on the results
reported in this subsection. The feasibility of  problem~\eqref{eq:optimization_problem_with_mean} for a given graph $G$ and threshold $t$ depends only on whether the
intersection of $\mathcal{C}(G)$ with the half-space
$\{\mathbf{C} : \frac{1}{2|E|}\langle \mathbf{C}-\mathbf{I}, \mathbf{1}_{p\times p}\rangle \geq t\}$
is empty, a property of the constraint set alone, independent of $\bar{\mathbf{C}}$
or the objective function. Any other choice of initial matrix would therefore yield
identical feasibility heatmaps.

We now describe precisely how each trial underlying
Figure~\ref{fig:feasibility_heatmaps} is generated. For every trial, we draw a
mean threshold $t$ uniformly from $[-1, 1]$ and independently generate a
random graph $G$ from the corresponding family (Erd\H{o}s--R\'{e}nyi,
Watts--Strogatz, Barab\'{a}si--Albert, or chordal/split, as described in
Section~\ref{subsec:graph_structure}) using that family's own generative
parameters (e.g., the target edge probability for Erd\H{o}s--R\'{e}nyi, or
the rewiring probability for Watts--Strogatz). Because these generative
parameters do not always translate into the realized edge density in a
one-to-one way, we do not bin trials by the target parameter directly;
instead, once $G$ has been sampled
, we compute its \emph{realized} edge
density $d = |E| / \binom{p}{2}$ and use the resulting pair $(d, t)$ to
locate the trial on the heatmap. Each pixel then aggregates all trials whose
$(d, t)$ falls within its bin, and is colored by the proportion of those
trials for which problem~\eqref{eq:optimization_problem_with_mean} is
feasible under the sampled graph and threshold. 

The observed patterns are consistent with the theoretical bounds:
\begin{itemize}
    \item \textbf{Erd\H{o}s--R\'{e}nyi at intermediate densities ($d \approx 0.5$):}
    This regime is the most constrained. In an Erd\H{o}s--R\'{e}nyi graph
    $G(p, d)$ with $d \approx 0.5$, two competing effects create the most
    constrained regime. At low densities, few cycles exist at all. At high
    densities, most cycles acquire chords and cease to be chordless. At
    intermediate densities, the graph is dense enough to contain many cycles,
    yet sparse enough that most of these cycles still lack chords, resulting in
    a large number of chordless cycles of length $k \geq 4$ that each impose
    independent upper bounds on the mean correlation. By the
    theoretical bounds of Section~\ref{subsec:threshold_bounds}, each such
    chordless cycle $C_k$ imposes an independent upper bound on the mean
    correlation of its $k$ edges --- at most $1/2$ for even $k$ and at most
    $1/(2\cos(\pi/k))$ for odd $k$. Since the mean constraint
    \eqref{eq:mean_constraint} is a global average over all edges of the graph,
    it is simultaneously constrained by every chordless cycle present. More
    precisely, for each chordless cycle $C_k$ in the graph, the sub-matrix of
    $\mathbf{C}$ indexed by the vertices of $C_k$ must itself be a valid
    correlation matrix in $\mathcal{C}(C_k)$, and thus satisfies the
    cycle-specific upper bound. The global mean threshold $t$ is therefore
    bounded above by the most restrictive combination of all such local cycle
    constraints, rendering high values of $t$ structurally infeasible at
    $d \approx 0.5$.

    \item \textbf{Erd\H{o}s--R\'{e}nyi at low densities ($d \to 0$):} The graph
    is extremely sparse and consists predominantly of trees and isolated edges.
    Tree-structured subgraphs contain no cycles, and therefore impose no cyclic
    bottlenecks on the PSD cone. In the absence of chordless cycles, the
    zero-constraints on non-edges do not generate the cross-edge coupling that
    limits the mean correlation, and the feasibility region expands to accommodate
    larger values of $t$.

    \item \textbf{Erd\H{o}s--R\'{e}nyi at high densities ($d \to 1$):} As the
    graph approaches the complete graph, every cycle of length $k \geq 4$ is
    increasingly likely to possess at least one chord. By definition, a chord
    breaks the chordless cycle into smaller cycles, and as established in
    Section~\ref{subsec:threshold_bounds}, chordal graphs are entirely free of
    the cycle-induced upper bounds on mean correlation. Consequently, the cyclic
    constraints are progressively relaxed as $d \to 1$, and the feasibility
    region expands toward larger values of $t$.

\item \textbf{Chordal graphs:} By definition, chordal graphs contain no
    chordless cycles of length $4$ or greater, so none of the cycle-imposed
    upper bounds of Section~\ref{subsubsec:threshold_bounds} apply. However,
    as established in Section~\ref{subsubsec:tridiagonal_bound}, the absence
    of chordless cycles does not by itself guarantee that the mean
    correlation can approach $1$. Sparse, path-like chordal graphs remain
    subject to the chain-sequence bound of~\eqref{eq:tridiag_bound}. Our
    chordal graphs are generated as random split graphs
    (Section~\ref{subsec:graph_structure}), which contain a clique of size
    $k$ growing with the edge density $d$; this clique structure, rather
    than chordality per se, is what allows the feasibility region to widen
    substantially as $d$ increases, consistent with the equicorrelation
    bound~\eqref{eq:equicorr_range}, which shows that mean correlations close
    to $1$ are achievable precisely when many vertices form a densely
    interconnected (near-complete) subset.
\end{itemize}

It is important to note that the light blue regions appearing in certain heatmaps
do not indicate infeasibility of problem~\eqref{eq:optimization_problem_with_mean},
but rather the absence of graph samples at those $(d, t)$ configurations. This
occurs because not all graph families can be generated at arbitrary edge densities.
For instance, the Watts--Strogatz model constructs graphs by rewiring a regular
ring lattice, which constrains the achievable density range; similarly, the
Barab\'{a}si--Albert model grows graphs via preferential attachment, which
intrinsically limits the range of densities that can be produced for a given
number of nodes $p$. Consequently, for density values outside the generatable
range of a given model, no random graph instance is available and no feasibility
trial is recorded, resulting in the light blue coloring. This is a sampling
artifact of the graph generation procedure and carries no implication about the
theoretical feasibility of the optimization problem at those densities.

\begin{figure}[!ht]
    \hspace{-20pt}
    \begin{tabular}{cccc}
        \includegraphics[width=0.24\textwidth]{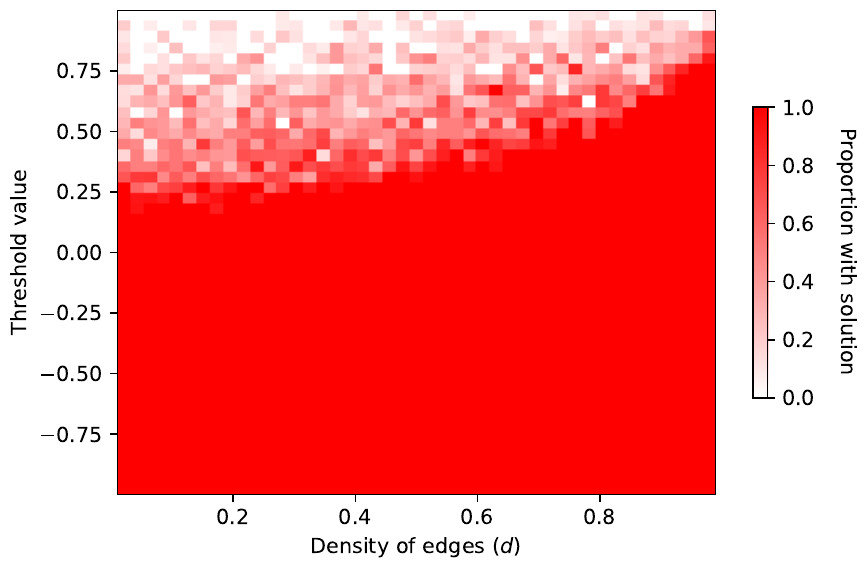} &
        \includegraphics[width=0.24\textwidth]{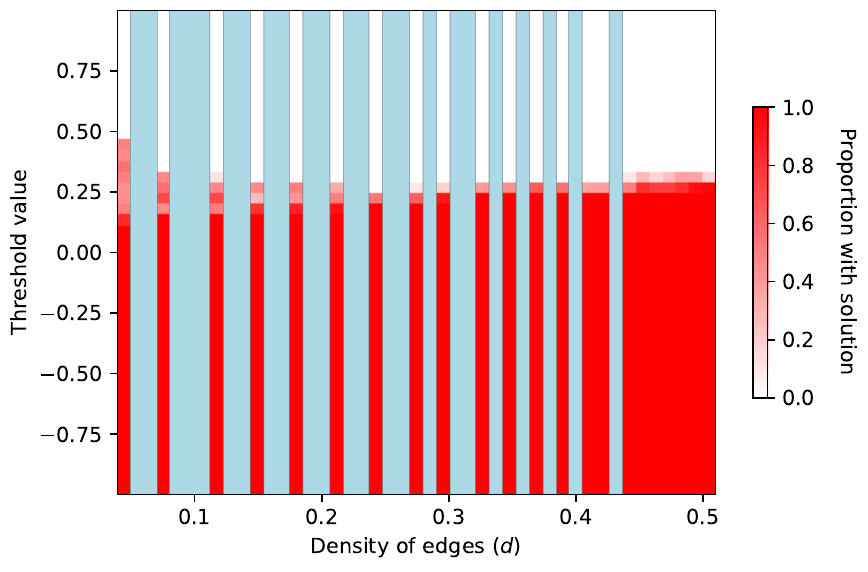} &
        \includegraphics[width=0.24\textwidth]{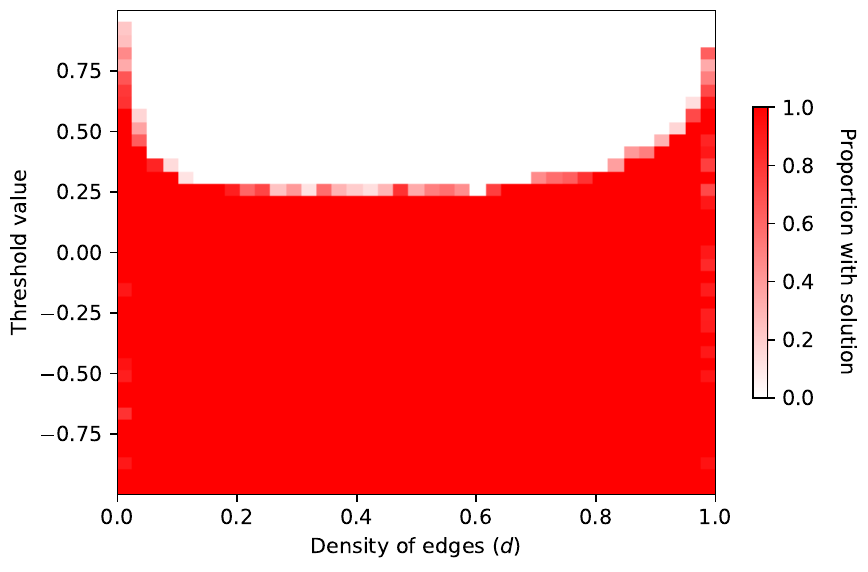} &
        \includegraphics[width=0.24\textwidth]{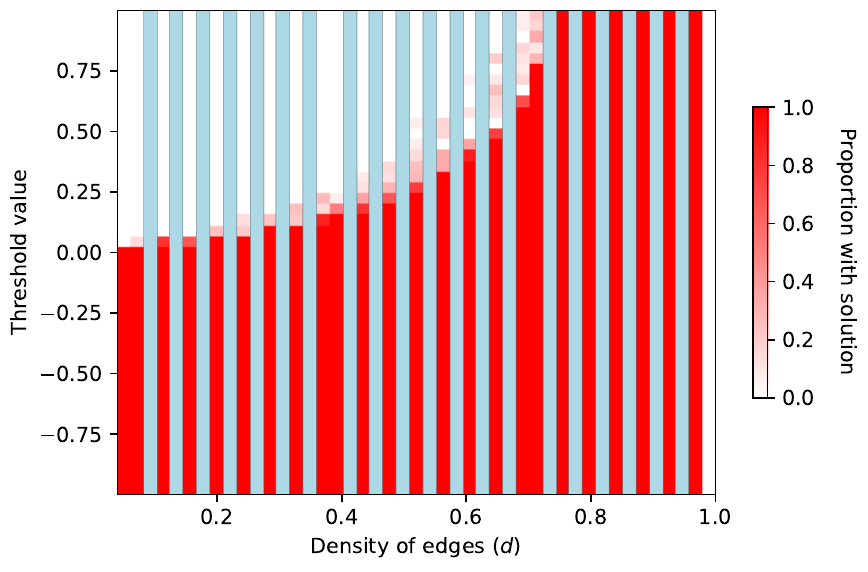} \\
        (a) Chordal & (b) Barab\'{a}si--Albert & (c) Erd\H{o}s--R\'{e}nyi & (d) Watts--Strogatz \\
    \end{tabular}
    \caption{Feasibility heatmaps for four graph models as a function of edge
    density $d$ (x-axis) and mean correlation threshold $t$ (y-axis). Each pixel
    encodes the proportion of trials, over 50 i.i.d.\ randomly generated graphs
    with $p = 51$ nodes, for which a feasible solution to
    problem~\eqref{eq:optimization_problem_with_mean} was found. Fully red pixels
    correspond to proportion 1 (always feasible), lighter pixels indicate partial
    or complete infeasibility, and light blue pixels indicate that no graph sample
    was available at that density for the corresponding model.}
    \label{fig:feasibility_heatmaps}
\end{figure}

\subsection{Real-World Applications}
\label{subsec:realworld}

To validate the practical utility of our framework, we apply it to two real-world
datasets where the distributional properties of correlation matrices are of direct
scientific or financial interest. We used Resting-state functional MRI data from
neuroscience, and stock return correlations from financial markets. In both cases,
empirical correlation matrices exhibit a pronounced positive shift in their
entry distributions.

\subsubsection{Neuroscience: Functional Brain Connectivity}
\label{subsubsec:neuroscience}

The empirical baseline for our neuroscience experiments is derived from rat
functional MRI (fMRI) recordings \cite{guillaume}. Rather than operating on raw
time series, our framework ingests localized frequency components. Specifically,
122 wavelet coefficients were extracted per region across $p = 51$ distinct brain
areas. These coefficients are obtained by applying a Daubechies order-8 discrete
wavelet transform to the original 30-minute resting-state recordings (sampled at
a 0.5-second repetition time for a total of $3{,}600$ time points). By isolating
wavelet scale 4, our analysis targets the functionally relevant $[0.06, 0.12]$~Hz
frequency band \cite{achard2022generation}.

\subsubsection{Financial Markets: S\&P 500}
\label{subsubsec:finance}

In financial simulation, we address the challenge of generating correlation
matrices derived from the S\&P 500 index \cite{Gautier2020}. Specifically, we
consider stocks and estimate pairwise correlations from their daily
returns over business days.

To validate our framework against a generative baseline, we compare against the
GAN-based approach of \cite{Gautier2020}. As the authors note, GAN-based
generation of correlation matrices suffers from scalability limitations (we restrict ourselves to $3 \times 3$ correlation matrices). To construct the training
dataset, three variables are selected at random from the full empirical matrix at
each iteration, for example, for the rat fMRI dataset ($p = 51$), random
$3 \times 3$ submatrices are extracted and used as training samples. This restriction to $3 \times 3$ matrices enables visualization within the
three-dimensional elliptope, as in the illustrative example of
Section~\ref{subsec:elliptope}.

Figures~\ref{fig:rat} and~\ref{fig:sp500} display, for the rat fMRI and S\&P 500
datasets respectively, three views: the elliptope (feasible PSD region), the
real data overlaid on the elliptope, and the GAN-generated samples. As visible in
both figures, the GAN-generated samples closely mimic the real data. Notably, the
two empirical datasets exhibit visually distinct geometric structures. The
financial market data is more spread throughout the elliptope, while the rat
fMRI data forms a denser, tree-like (spider-shaped) cluster near the boundary.
This contrast is faithfully reproduced by the GAN in both cases.

\begin{figure}[!ht]
    \centering
    \begin{subfigure}[b]{0.45\linewidth}
        \centering
        \includegraphics[width=0.9\linewidth]{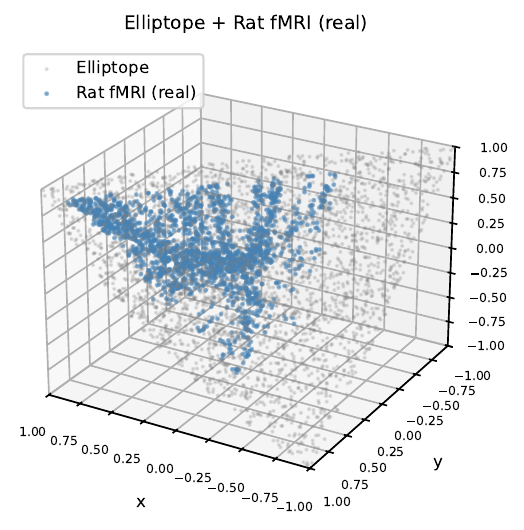}
        \caption{Elliptope and real data (training data)}
    \end{subfigure}
   \hspace{1cm}
    \begin{subfigure}[b]{0.45\linewidth}
        \centering
        \includegraphics[width=0.9\linewidth]{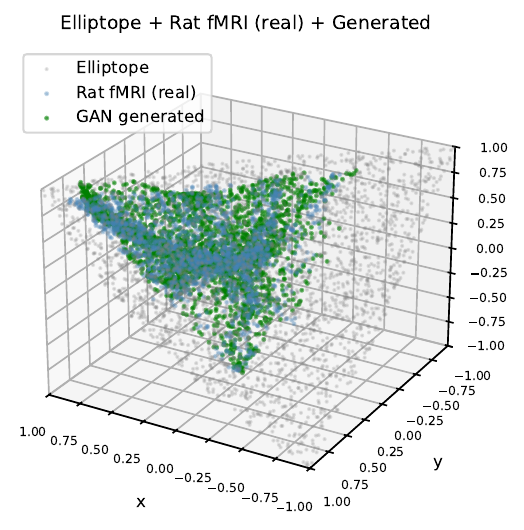}
        \caption{Elliptope, real data (training data) and generated data}
    \end{subfigure}
    \caption{GAN correlation generation trained on rat fMRI data \cite{Gautier2020}.
    Axes represent the off-diagonal entries of
    a $3 \times 3$ correlation matrix. The elliptope (feasible region,
    gray). \emph{(a)}: real rat fMRI data subset (blue) overlaid on the elliptope.
    \emph{(b)}: GAN-generated samples (green) overlaid on the elliptope and real
    data.}
    \label{fig:rat}
\end{figure}

\begin{figure}[!ht]
    \centering
    \begin{subfigure}[b]{0.45\linewidth}
        \centering
        \includegraphics[width=0.9\linewidth]{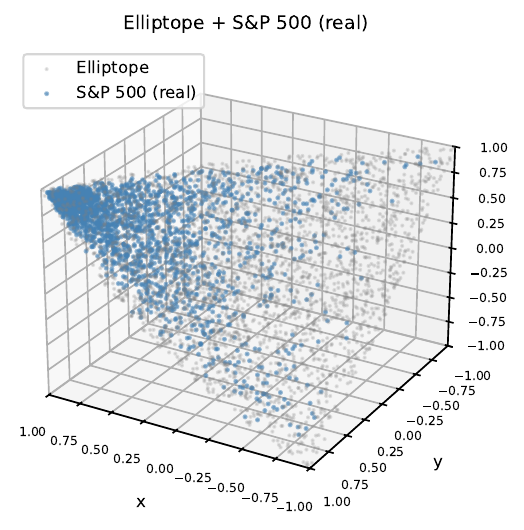}
        \caption{Elliptope and real data (training data)}
    \end{subfigure}
    \hspace{1cm}
    \begin{subfigure}[b]{0.45\linewidth}
        \centering
        \includegraphics[width=0.9\linewidth]{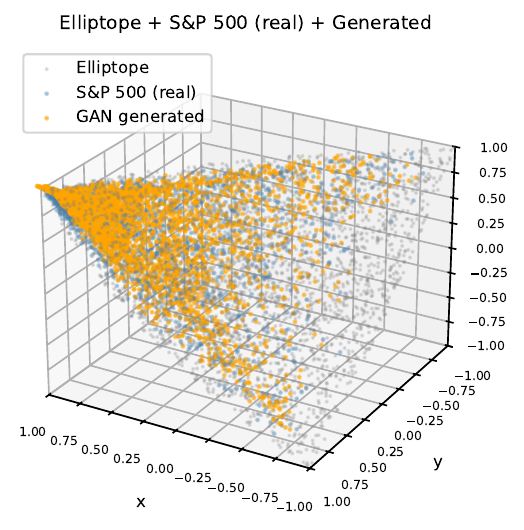}
        \caption{Elliptope, real data (training data) and generated data}
    \end{subfigure}
    \caption{GAN correlation generation trained on S\&P 500 data \cite{Gautier2020}.
    Layout identical to Figure~\ref{fig:rat}. The elliptope (feasible
    region, gray). \emph{(a)}: real S\&P 500 correlation data subset (blue).
    \emph{(b)}: GAN-generated samples (orange) overlaid on the elliptope and real
    data; as in the neuroscience case, several generated points fall outside the
    elliptope.}
    \label{fig:sp500}
\end{figure}

Figure~\ref{fig:our_method} shows the output of our method for a complete graph with three nodes in $\mathbb{R}^3$, where $\bar{\mathbf{C}}$ is drawn uniformly
from the elliptope $\mathcal{C}$ rather than entrywise on $[-1,1]$, since uniform-over-the-elliptope
sampling admits a simple constant density with respect to the
Lebesgue measure, which yields the explicit projection formula. The matrix
$\mathbf{C}_{optimal}$ is obtained by solving
problem~\eqref{eq:optimization_problem_with_mean} under a mean
threshold of $t = 0.2$.

The projected solutions in Figure~\ref{fig:our_method} exhibit a striking concentration on the
hyperplane boundary, visible as
a dense ring. This can be explained
rigorously via a Radon--Nikodym decomposition of the push-forward
measure induced by the projection map, which splits into an
absolutely continuous part and a
singular part; see \ref{app:radon_nikodym} for the full derivation.

\begin{figure}[!ht]
        \centering
        \includegraphics[width=0.3\linewidth]{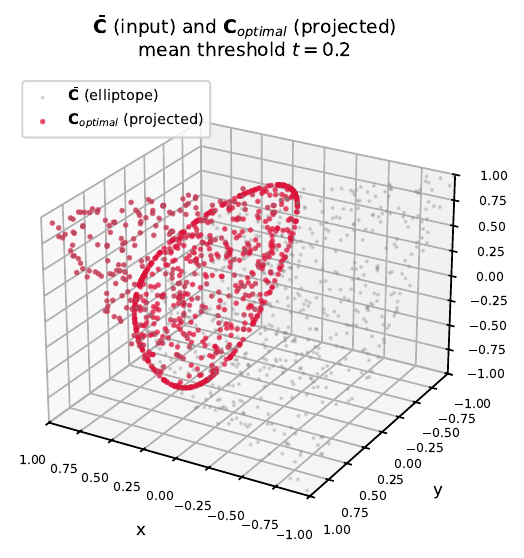}
        \caption{Generated correlation matrices via 
        problem~\eqref{eq:optimization_problem_with_mean} with mean 
        threshold $t = 0.2$, visualized inside the three-dimensional 
        elliptope. Gray points represent initial matrices 
        $\bar{\mathbf{C}}$ drawn uniformly from $\mathcal{C}$ 
        (the full elliptope), and red points are the corresponding 
        optimal projections $\mathbf{C}_{\mathrm{optimal}} \in 
        \mathcal{C}(G)$. Axes $x, y, z$ correspond to the three 
        off-diagonal entries $c_{12}, c_{13}, c_{23}$ of a $3\times 3$ 
        correlation matrix.}
    \label{fig:our_method}
\end{figure}

\subsection{Generation in Higher Dimension}
\label{subsec:higher_dim}

We now move beyond the illustrative $p=3$ case of Section~\ref{subsec:feasibility}
and validate our framework at a realistic problem size, using the same two
real-world datasets introduced above, resting-state functional MRI recordings
from neuroscience \cite{guillaume} and daily stock return correlations from the
S\&P~500 index \cite{Gautier2020}. The rat fMRI dataset has a native dimension of
$p = 51$ brain regions. To keep the two application domains directly comparable and since the S\&P~500 universe contains many more constituents than the
rat dataset has regions, we restrict the financial dataset to the first
$51$ assets as well, so that every experiment in this subsection, and the
remainder of Section~\ref{sec:applications}, is run at the same dimension
$p = 51$ for both datasets.

\subsubsection{Perturbing an Estimated Correlation Matrix}
\label{subsubsec:noise_robustness}

Randomly perturbing a baseline correlation matrix, known as \emph{correlation
matrix stress testing}, is well studied in risk management. Naively
perturbing individual entries almost always breaks positive
semidefiniteness, so existing methods restore validity by
construction, e.g.\ via a hyperspherical/Cholesky-angle reparametrization
\cite{rebonato2000most}.
These methods target local, scenario-driven stressing of a baseline matrix,
not generation under a prescribed graph structure or target moments, which
is our focus here.

We first assess how sensitive our projection-based construction is to noise
corruption of the input correlation matrix. Starting from the empirical
$51 \times 51$ correlation matrix $\bar{\mathbf{C}}$ of each dataset, we
perturb its off-diagonal entries with zero-mean Gaussian noise of standard
deviation $\sigma$, keeping the result symmetric, and then project the noisy
matrix back onto the PSD cone. This is precisely the unconstrained special
case of the dual approach of Section~\ref{subsec:malick_dual}. When no linear
equality or inequality constraints are active (no zero pattern, no unit-diagonal
requirement beyond what is already present), the dual problem \eqref{eq:dual_our_problem} is unconstrained and
its solution is realized directly by a single eigenvalue decomposition,
\begin{equation}
    \mathbf{C}_{\sigma}^{\mathrm{proj}}
    = \operatorname*{arg\,min}_{\mathbf{X} \succeq 0}
    \tfrac{1}{2}\|\mathbf{X} - \mathbf{C}_{\sigma}^{\mathrm{noisy}}\|_F^2
    = \operatorname{proj}_{\succeq 0}\!\bigl(\mathbf{C}_{\sigma}^{\mathrm{noisy}}\bigr),
    \label{eq:noise_projection}
\end{equation}
where $\mathbf{C}_{\sigma}^{\mathrm{noisy}}$ denotes the noise-corrupted matrix
at noise level $\sigma$ and $\operatorname{proj}_{\succeq 0}(\cdot)$ clips
negative eigenvalues to zero. Figure~\ref{fig:noise_overlay} overlays, for each
dataset and for $\sigma \in \{0.05, 0.15, 0.30\}$, the density of the
off-diagonal entries of the real data, the noisy matrix, and its PSD
projection~\eqref{eq:noise_projection}. As expected, the projection pulls the
noisy distribution back toward the real-data density, with the correction
becoming more pronounced as $\sigma$ grows.

\begin{figure}[!ht]
    \centering
    \begin{subfigure}[b]{0.48\textwidth}
        \centering
        \includegraphics[width=\textwidth]{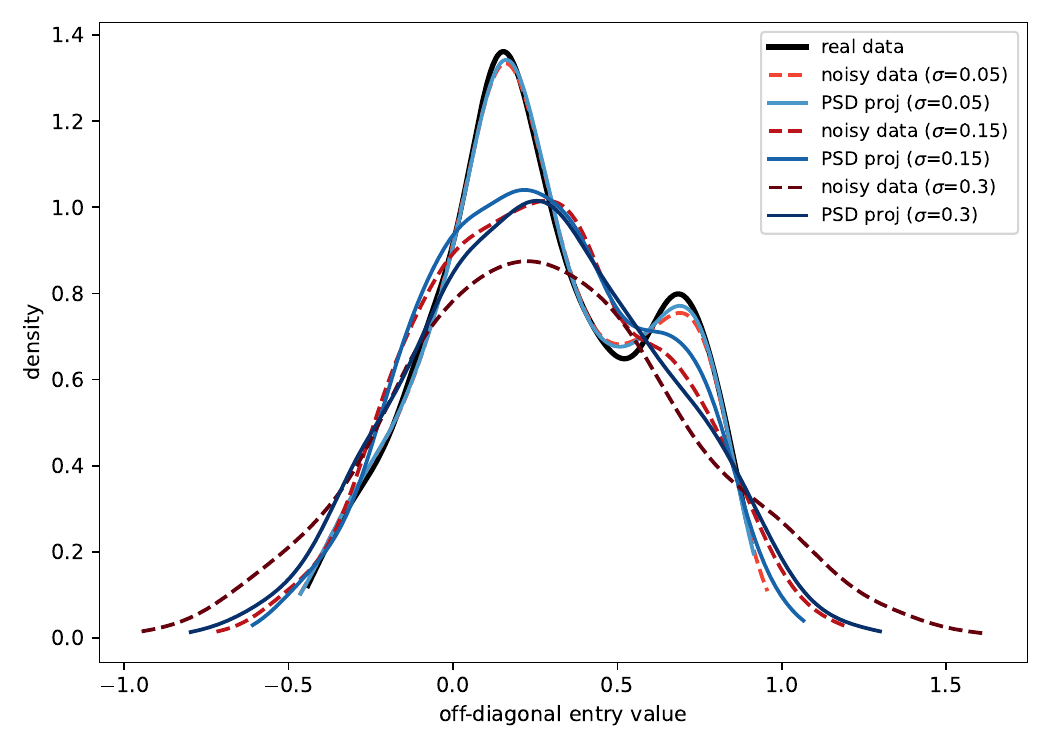}
        \caption{Rat fMRI ($p=51$).}
        \label{fig:noise_rat}
    \end{subfigure}
    \hfill
    \begin{subfigure}[b]{0.48\textwidth}
        \centering
        \includegraphics[width=\textwidth]{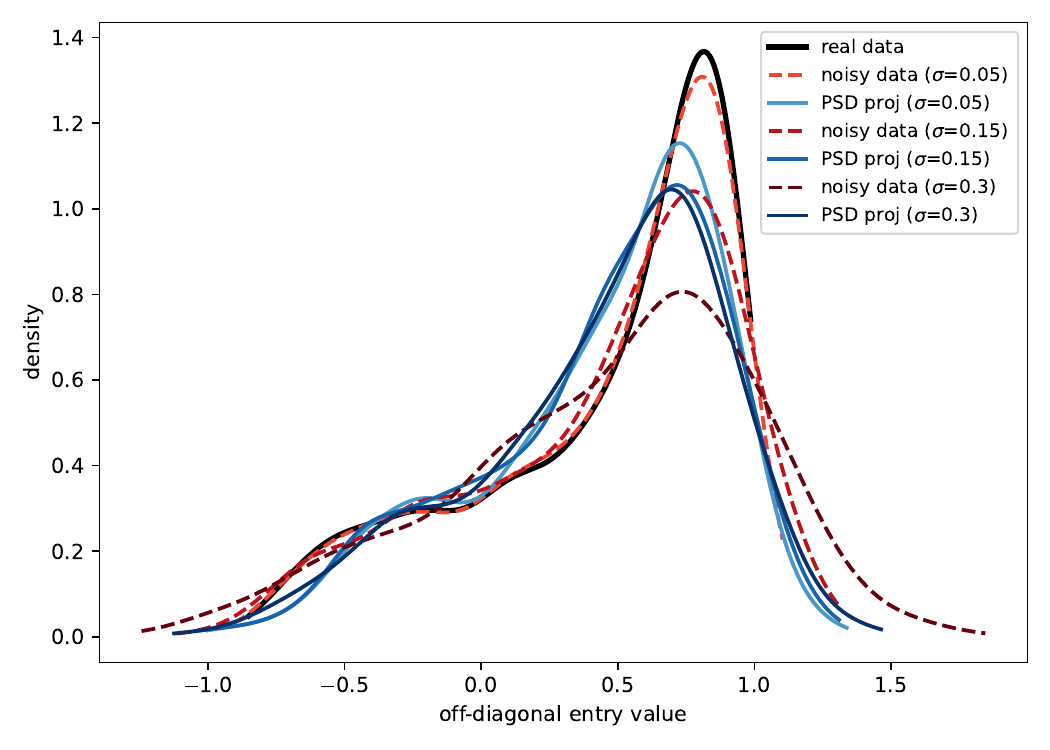}
        \caption{S\&P~500, first 51 assets.}
        \label{fig:noise_sp500}
    \end{subfigure}
    \caption{Density of off-diagonal correlation entries for the real data, the
    Gaussian-noise-corrupted matrix, and its PSD projection
    \eqref{eq:noise_projection}, at noise levels
    $\sigma \in \{0.05, 0.15, 0.30\}$, for \emph{(a)} the rat fMRI dataset and
    \emph{(b)} the S\&P~500 dataset (first 51 assets). Higher $\sigma$
    increasingly flattens and widens the noisy distribution, while the PSD
    projection consistently pulls it back toward the real-data density.}
    \label{fig:noise_overlay}
\end{figure}

\subsubsection{Effect of Initialization}
\label{subsubsec:init_moment_constraints}

We next examine how the choice of initialization $\bar{\mathbf{C}}$ affects
the solution of the mean-constrained projection
problem~\eqref{eq:optimization_problem_with_mean}, fitted to the empirical
mean $t_{\mathrm{mean}}$ of each dataset's off-diagonal entries. For each
dataset, we consider the two random initializations already introduced in
Section~\ref{sec:framework}: an \emph{entrywise-uniform} initialization, with
lower-triangular entries drawn i.i.d.\ from $\mathrm{Uniform}(-1,1)$ and unit
diagonal, and an \emph{elliptope-uniform} initialization, sampled uniformly
over the space of correlation matrices $\mathcal{C}$ via the onion method of
\cite{lewandowski2009generating}. For each initialization, we solve
problem~\eqref{eq:optimization_problem_with_mean} with $t = t_{\mathrm{mean}}$
via CVXPY.

Beyond the mean, our convex framework readily accommodates further moment
constraints on the generated off-diagonal entries in order to more closely
mimic additional distributional characteristics of real data, for
instance, a target variance can be added as a convex quadratic sublevel-set
constraint without changing the solvers of Section~\ref{sec:solvers}. We do
not develop such extensions further here; we only note that not every
moment admits a convex formulation (e.g., a target skewness is a cubic
functional of the entries), so incorporating higher-order moments beyond the
mean and variance would require departing from the convex framework.

Figures~\ref{fig:scenario_rat_combined}--\ref{fig:scenario_sp500_combined}
display, for each dataset and each initialization, the density of the
generated off-diagonal entries under the mean constraint, together with the
initialization and the real-data densities for reference.
In both cases, the mean constraint successfully pulls the solution's mode toward that of the real data, confirming that a single linear constraint on the first moment is enough to correct the gross positive shift missed by an unconstrained initialization, but the resulting shape still depends on the starting distribution $\bar{\mathbf{C}}$, since the objective in \eqref{eq:optimization_problem_with_mean} minimizes distance to $\bar{\mathbf{C}}$ rather than to the real-data density itself.

\begin{figure}[!ht]
    \centering
    \begin{subfigure}[t]{0.45\linewidth}
        \centering
        \includegraphics[width=0.8\linewidth]{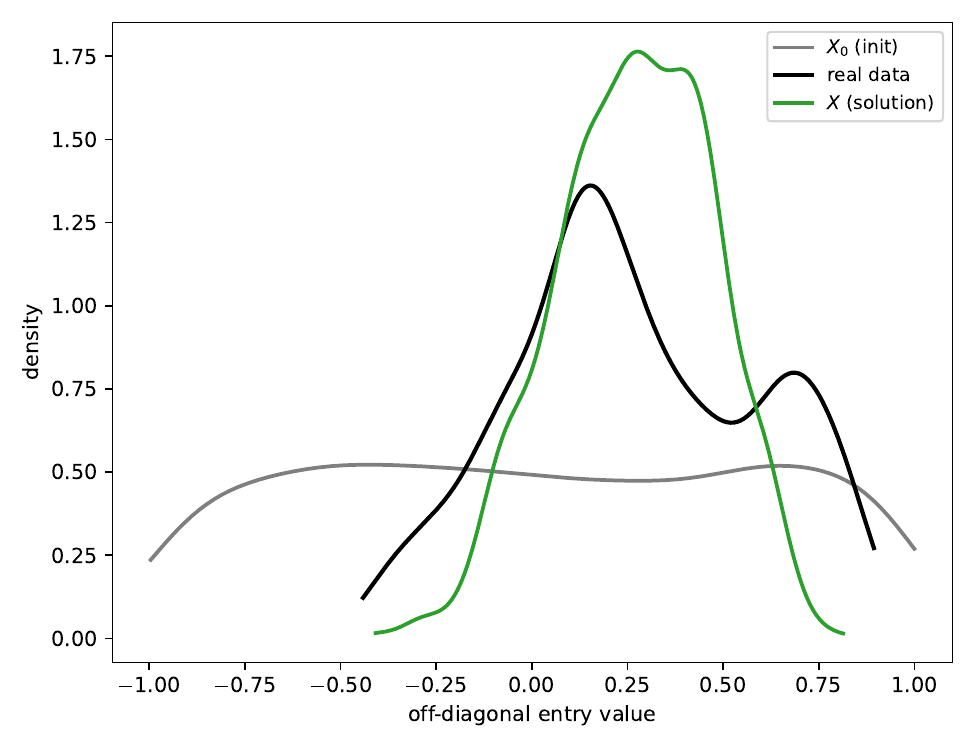}
        \caption{Mean (entrywise-uniform).}
    \end{subfigure}
    \hspace{1cm}
    \begin{subfigure}[t]{0.45\linewidth}
        \centering
        \includegraphics[width=0.8\linewidth]{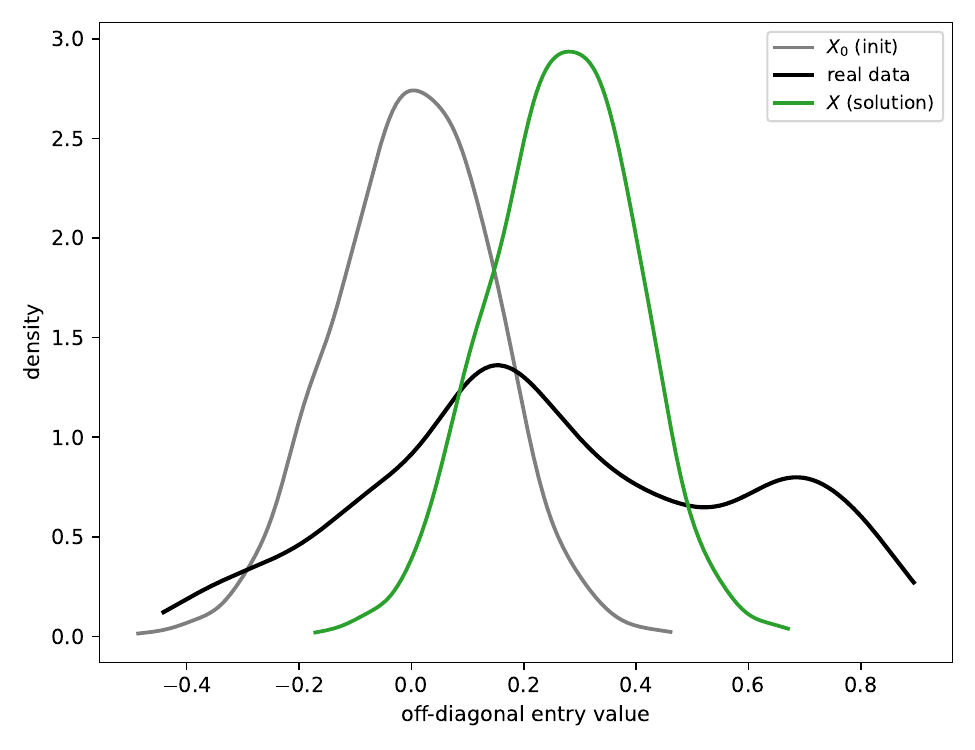}
        \caption{Mean (elliptope-uniform).}
    \end{subfigure}
    \caption{Rat fMRI dataset initialization scenarios under the mean
    constraint: density of off-diagonal entries of the initialization
    $\bar{\mathbf{C}}$, the real data, and the solution $\mathbf{C}$ of
    problem~\eqref{eq:optimization_problem_with_mean} with $t=t_{\mathrm{mean}}$.
    Left: entrywise-uniform initialization. Right: elliptope-uniform
    initialization, sampled uniformly over $\mathcal{C}$.}
    \label{fig:scenario_rat_combined}
\end{figure}

\begin{figure}[!ht]
    \centering
    \begin{subfigure}[t]{0.45\linewidth}
        \centering
        \includegraphics[width=0.8\linewidth]{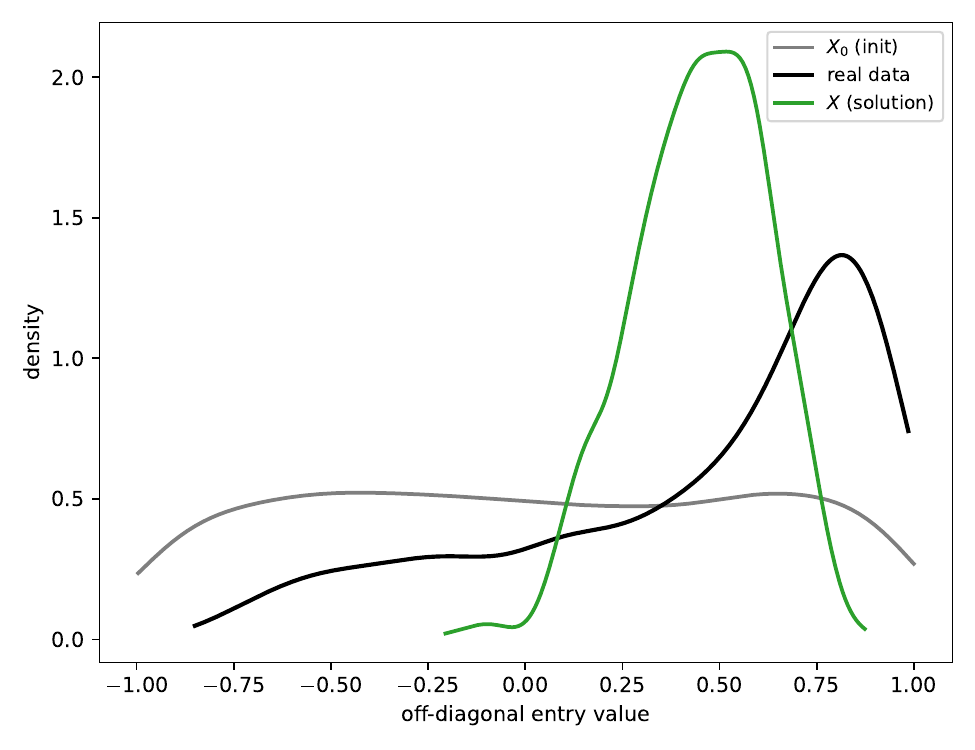}
        \caption{Mean (entrywise-uniform).}
    \end{subfigure}
    \hspace{1cm}
    \begin{subfigure}[t]{0.45\linewidth}
        \centering
        \includegraphics[width=0.8\linewidth]{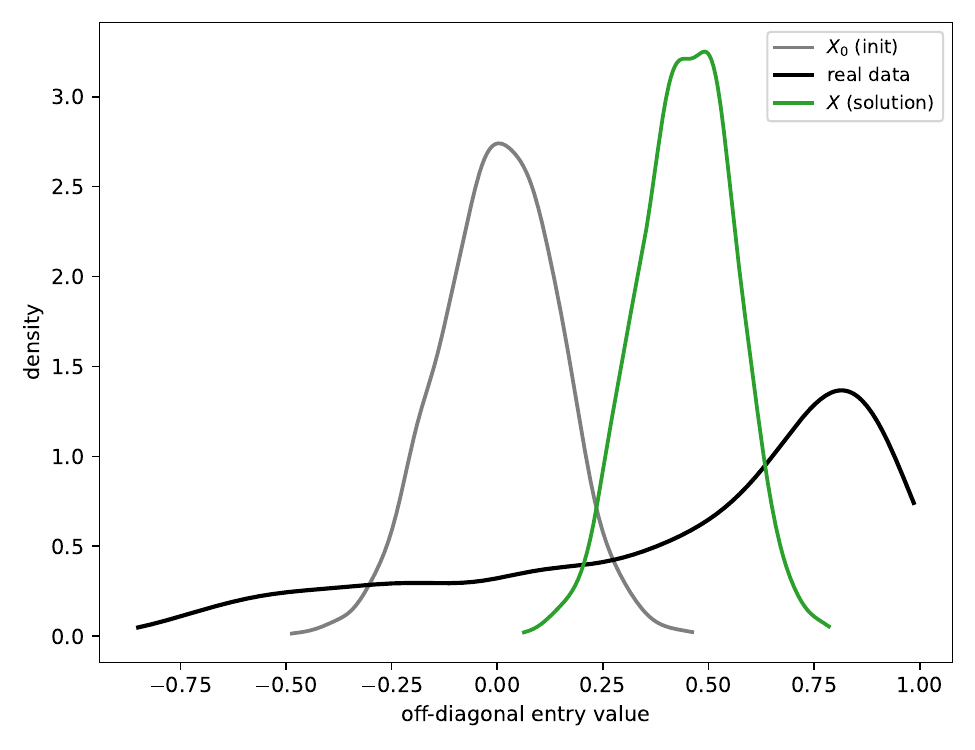}
        \caption{Mean (elliptope-uniform).}
    \end{subfigure}
    \caption{S\&P 500 dataset (first 51 assets) initialization scenarios
    under the mean constraint, same layout as
    Figure~\ref{fig:scenario_rat_combined}.}
    \label{fig:scenario_sp500_combined}
\end{figure}

\subsection{Effect of Graph Structure on Real Data}
\label{subsec:graph_effect_real}

We now reintroduce the graph-sparsity constraint of
Section~\ref{sec:preliminaries} into the two real-world applications. Rather
than solving the fully dense projection problem
\eqref{eq:optimization_problem_with_mean} as in
Section~\ref{subsec:higher_dim}, we now impose $c_{ij} = 0$ for pairs
$(i,j) \notin E$, so that the generated matrix respects a prescribed zero
pattern in addition to matching the target moments.

\subsubsection{Graph Estimated Directly From the Data}
\label{subsubsec:graph_real_data}

We first estimate the graph $G$ directly from each empirical correlation
matrix by retaining the $50\%$ of off-diagonal entries with the largest
absolute value (edge density $d = 0.5$) as edges. We then solve
problem~\eqref{eq:optimization_problem_with_mean} under this estimated graph
and a mean threshold $t$ matched to the empirical mean of the retained
entries. Simply zeroing the remaining $50\%$ of off-diagonal entries of the
empirical matrix does not, in general, preserve positive semi definiteness, the
resulting masked matrix typically has negative eigenvalues. For reference, we
also report the density of this \emph{masked} real matrix (the real
off-diagonal entries restricted to the estimated edge set, prior to any
projection), alongside the density of the full (unmasked) real matrix and of
the two projected solutions, with and without the mean constraint.

Figure~\ref{fig:real_vs_proj_rat} shows the density of the non-zero,
off-diagonal entries of a $51 \times 51$ sample correlation matrix for each
scenario. The solid black line represents a sample of the full real data,
while the resulting distribution for the masked real data at a $0.5$ edge
density (which is not necessarily PSD) is shown as the dotted orange line.
Because masking some entries to zero typically destroys the PSD property,
the projected solutions (shown in green and dotted red, corresponding to the
zero-only projection and the mean-constrained projection in
problems~\eqref{eq:optimization_problem} and~\eqref{eq:optimization_problem_with_mean},
respectively) must adjust their values to restore positive semidefiniteness;
hence, they are not an exact match with the masked real data (the orange line).
The main message of this figure is further demonstrated using a synthetic baseline:
if we initialize a matrix with uniform entries between $-1$ and $1$
(dotted purple line), the result of the optimization problem under the exact
same zero pattern and mean threshold constraint (dotted blue line) mimics the distribution of the masked real data (dotted orange line).

For the S\&P~500 dataset, the analogous experiment instead exposes a
structural limitation rather than a fit. As shown in
Figure~\ref{fig:real_vs_proj_sp500}, problem~\eqref{eq:optimization_problem_with_mean}
admits \emph{no feasible solution} once both the estimated graph structure
($d=0.5$) and the matched mean threshold of $t = 0.75$ are imposed jointly. The
zero-only projection (no mean constraint) remains feasible and is shown for
reference. This contrast between the two datasets illustrates a broader
point. The mean constraint is not always compatible with a given zero
pattern, and its feasibility must be checked jointly with $G$ rather than
imposed in isolation. For the S\&P~500 data, its empirical correlation
structure at $d=0.5$ is too strongly and too broadly correlated for any
matrix in $\mathcal{C}(G)$ to simultaneously satisfy the graph constraint
and match the empirical mean.

\begin{figure}[!ht]
    \centering
    \begin{subfigure}[t]{0.45\textwidth}
        \centering
        \includegraphics[width=0.8\textwidth]{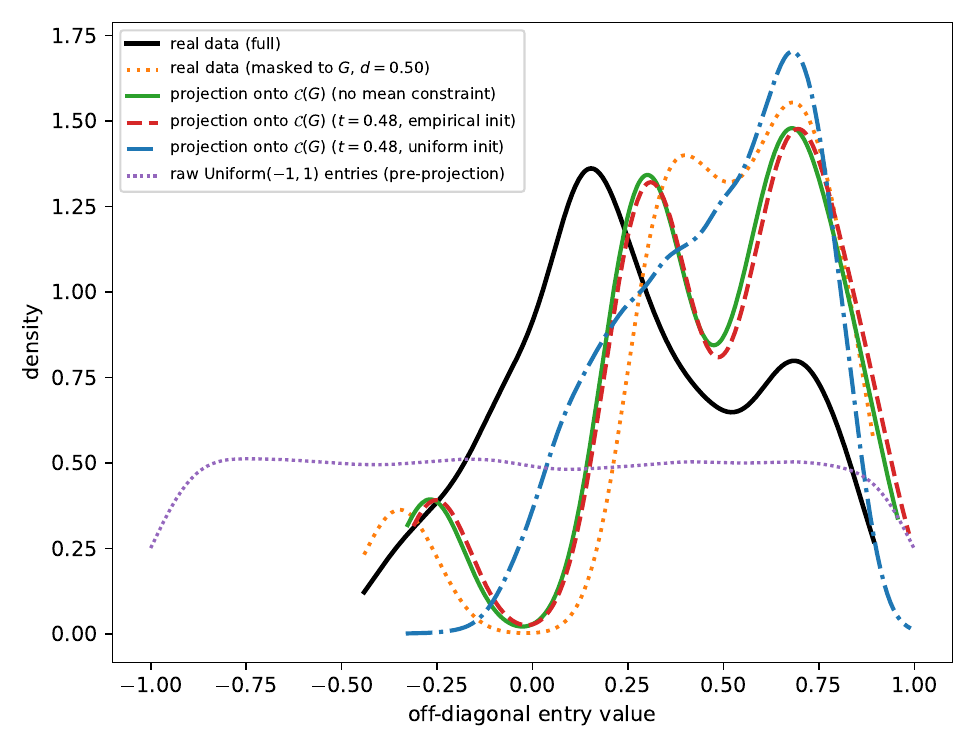}
        \caption{Rat fMRI ($p=51$): density of the full real off-diagonal
        and non-zero entries, the masked real entries (restricted to the
        estimated edge set at $d=0.5$, prior to projection), the projection
        onto $\mathcal{C}(G)$ with only the zero constraint, and the projection
        under the matched mean threshold $t=t_{\mathrm{mean}}$. Synthetic 
        uniform initialization and its projection are also shown to demonstrate 
        distribution mimicking.}
        \label{fig:real_vs_proj_rat}
    \end{subfigure}
    \hspace{1cm}
    \begin{subfigure}[t]{0.45\textwidth}
        \centering
        \includegraphics[width=0.8\textwidth]{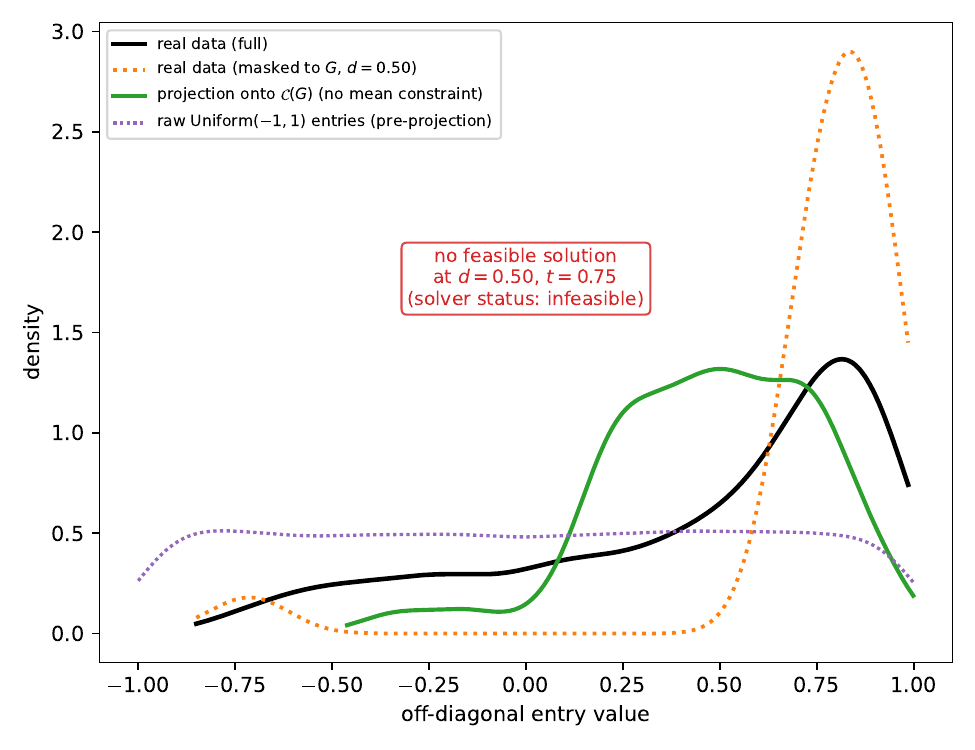}
        \caption{S\&P~500 (first 51 assets), same layout as panel (a). The
        mean-constrained projection is infeasible at $d=0.5$ under the
        empirical mean threshold; only the zero-only projection is
        shown, alongside the full and masked real data.}
        \label{fig:real_vs_proj_sp500}
    \end{subfigure}
    \caption{Real vs.\ projected correlation-entry densities under an
    estimated graph structure ($d=0.5$), for \emph{(a)} the rat fMRI
    dataset and \emph{(b)} the S\&P~500 dataset (first 51 assets).}
    \label{fig:real_vs_proj_combined}
\end{figure}

\subsubsection{Effect of Graph Family Under a Mean Threshold}
\label{subsubsec:graph_family_real}

Beyond a single graph estimated from the data, we also examine how the
\emph{family} of graph affects the generated distribution when a mean
threshold is imposed. Figure~\ref{fig:mean_constraint_distribution} reports the
empirical densities of the non-diagonal, non-zero entries obtained by our
method across the five random graph families of
Section~\ref{subsec:graph_structure}, under a mean threshold $t = 0.2$ and
edge density $d = 0.5$, averaged over 50 i.i.d.\ randomly generated graphs with
$p = 51$ nodes. Without the mean constraint, the entry distributions of the
five families are visually similar for a given density $d$ (Figure~\ref{fig:fig3}). Once the mean threshold $t=0.2$ is imposed, however,
the picture changes markedly, enforcing the mean constraint successfully
shifts the bulk of the generated coefficients into the positive regime,
closely tracing the target empirical profile. The choice of graph family
now has a clearly visible effect on the resulting distribution. This
contrast is explained by the feasibility bounds of
Section~\ref{subsec:threshold_bounds}: once the mean threshold $t=0.2$ is
imposed, each family is limited by its own graph-dependent feasibility
ceiling, since the number and length of chordless cycles present in a given
family's graphs determine how close that family already sits to its
structural upper bound on the mean correlation. Families with fewer or
longer chordless cycles (and hence a looser ceiling) can reach $t=0.2$ with
only a mild redistribution of entries, while families whose graphs contain
many short chordless cycles must redistribute their entries far more
aggressively to satisfy the same threshold, producing markedly different
resulting distributions across families.

Finally, the feasibility bounds established in Section~\ref{subsec:threshold_bounds} and illustrated in Figure~\ref{fig:feasibility_heatmaps} can be used diagnostically to assess the plausibility of a given graph family as the generative structure underlying empirical data. If the empirical mean correlation exceeds a threshold $t$ and the estimated graph density exceeds $d$, and the pair $(t, d)$ lies outside the feasible region for a given graph family in our simulations, this suggests that the empirical data is unlikely to have been generated under that graph family's structural assumptions. For instance, as illustrated in Figure~\ref{fig:feasibility_heatmaps}, the combination of mean $t = 0.5$ and density $d = 0.5$ is infeasible under the Erdős–Rényi model, so empirical data exhibiting these characteristics is unlikely to arise from an Erdős–Rényi-type dependency structure. 

\begin{figure}[!ht]
    \centering
    \includegraphics[width=0.45\linewidth]{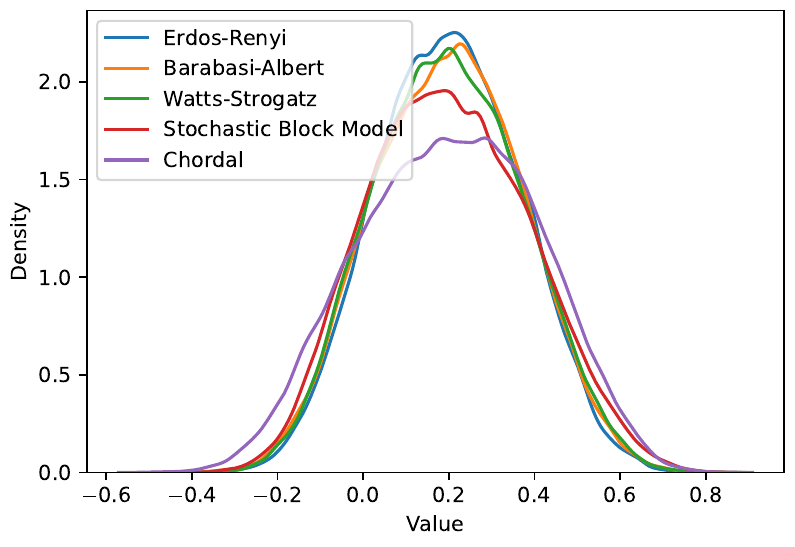}
    \caption{Empirical densities of non-diagonal, non-zero entries in generated
    correlation matrices under a mean threshold constraint $t = 0.2$ and edge
    density $d = 0.5$, estimated over 50 i.i.d.\ randomly generated graphs with
    $p = 51$ nodes, across the five graph families of
    Section~\ref{subsec:graph_structure}. Our method with an explicit mean
    constraint successfully reproduces the positive shift observed in
    empirical fMRI data, whereas baseline methods remain tightly centered
    around zero.}
    \label{fig:mean_constraint_distribution}
\end{figure}
\section{Conclusion}
\label{sec:conclusion}

In this paper, we introduced a convex optimization framework for generating
structured correlation matrices that are compatible with arbitrary graph
structures. By formulating the task as a projection problem in the Frobenius
sense, we guarantee a unique, node-ordering-independent solution that avoids
the systematic shrinkage toward zero associated with existing diagonal dominance
methods.

We demonstrated that our approach effectively accommodates real-world
distributional requirements through an additional linear mean-threshold
constraint, allowing for the generation of matrices that replicate the positive
shift observed in fMRI brain connectivity and financial market data. Through extensive numerical experiments, we compared three distinct
optimization solvers. CVXPY offers the greatest flexibility for prototyping
and auxiliary constraints, while both the dual approach and QSDPNAL scale to
high-dimensional problems ($p \geq 1000$) where CVXPY becomes impractical. In
our benchmarks (Table~\ref{tab:computation_times}) the dual approach was in
fact the fastest of the three at every tested dimension, including $p=1000$
(10.80 s vs.\ 48.39 s for QSDPNAL). We also
provided a rigorous characterization of the feasibility of the mean constraint as
a function of graph density and structure, proving that chordless cycles impose
hard structural upper bounds on the achievable mean correlation.

As summarized in Table~\ref{tab:comparison}, our method is the only approach that
simultaneously handles non-chordal graphs, avoids near-zero entries, supports
distributional control via the mean constraint, and is independent of node
ordering, at the cost of increased computational time relative to diagonal
dominance or partial orthogonalization.

\begin{table}[ht]
\centering
\caption{Qualitative comparison of methods for generating structured correlation matrices.}
\label{tab:comparison}
\resizebox{0.5\textwidth}{!}{
\begin{tabular}{lcccc}
\hline
\textbf{Property} & 
\textbf{\parbox{1.8cm}{\centering Diag.\\Dom.~\cite{cordoba2020generating}}} & 
\textbf{\parbox{1.8cm}{\centering Part.\\Orth.~\cite{cordoba2018metropolis}}} & 
\textbf{\parbox{2cm}{\centering Cholesky-\\based~\cite{cordoba2018metropolis}}} & 
\textbf{Ours} \\
\hline
Non-chordal graphs     & \checkmark & \checkmark & $\times$   & \checkmark \\
Mean control           & $\times$   & $\times$   & $\times$   & \checkmark \\
Order-independent      & \checkmark & $\times$   & $\times$   & \checkmark \\
Match real-data dist.\ & $\times$   & $\times$   & $\times$   & \checkmark \\
Low computational cost & \checkmark & \checkmark & \checkmark & $\times$   \\
\hline
\end{tabular}
}
\end{table}


Future work will focus on integrating more complex higher-order structural
constraints, such as spectral density requirements, and further optimizing the
dual approach for distributed computing environments.

\section*{Acknowledgements}

We thank J\'{e}r\^{o}me Malick for helpful discussions, and Moritz M\"{u}ller
and Aur\'{e}lie Lagoutte for fruitful discussions on the implementation of
chordal graphs. This work was supported by the Agence Nationale de la
Recherche under the France 2030 programme, reference ANR-23-IACL-0006.

\appendix

\section{Derivation of the Bimodal Density under Diagonal Dominance}

\label{app:bimodal}
To analyze the resulting distribution, consider the simplified symmetric case
$r_i = r_j = r$, where $r = \sum_{\substack{\ell \neq j,\, \ell \sim i}} |\tilde{c}_{i\ell}|$
collects the contributions of all other neighbors of node $i$. 

\textit{Initialization and the transformation map:} We begin by sampling the 
lower triangular entries of the unnormalized matrix $\tilde{\mathbf{C}}$ independently 
and uniformly from $[-1, 1]$, with upper triangular entries imposed by symmetry,
$\tilde{c}_{ij} \sim \mathrm{Uniform}(-1,1)$ for $i > j$, and $\tilde{c}_{ji} = \tilde{c}_{ij}$. 
The diagonal dominance procedure then applies the nonlinear normalization map
\begin{equation}
    \phi(\tilde{c}) = \frac{\tilde{c}}{\sqrt{\tilde{c}_{ii}\tilde{c}_{jj}}},
\end{equation}
which, in the simplified symmetric case where $\tilde{c}_{ii} = \tilde{c}_{jj} = |\tilde{c}| + r$, 
reduces to:
\begin{equation}
    \phi(\tilde{c}) = \frac{\tilde{c}}{|\tilde{c}| + r}, \qquad
    \phi'(t) = \frac{r}{(t+r)^2}, \quad t > 0.
    \label{eq:phi_map}
\end{equation}
The map $\phi$ is strictly decreasing and nonlinear, it compresses the support 
of the original entries by the factor $|\tilde{c}| + r$ in the denominator.

\textit{Resulting distribution:} To determine the distribution of the normalized 
entries $c_{ij} = \phi(\tilde{c}_{ij})$, we apply the change-of-variables formula. 
Since $\tilde{c}_{ij} \sim \mathrm{Uniform}(-1, 1)$ has density $f_{\tilde{c}}(t) = 1/2$, 
the density of $c_{ij} = x$ for $x > 0$ is
\begin{equation}
    f_{c}(x) = f_{\tilde{c}}(\phi^{-1}(x)) \cdot \left|\frac{d\phi^{-1}}{dx}\right| 
    = \frac{1}{2} \cdot \frac{1}{\phi'(\phi^{-1}(x))}.
    \label{eq:change_of_vars}
\end{equation}
Inverting the map \eqref{eq:phi_map} yields $\phi^{-1}(x) = rx/(1-x)$, so that
\begin{equation}
    \phi'(\phi^{-1}(x)) = \frac{r}{\left(\frac{rx}{1-x} + r\right)^2} = \frac{(1-x)^2}{r}.
\end{equation}
Substituting back into \eqref{eq:change_of_vars}, the density becomes
\begin{equation}
    f_c(x) \propto \frac{r}{(1-x)^2}, \quad x \in (0, \phi(1)) = \left(0, \frac{1}{1+r}\right).
\end{equation}
This density is strictly increasing on its support $(0, 1/(1+r))$, with maximum 
at the boundary $x = 1/(1+r)$. By symmetry of the initialization, the distribution 
on $(-1/(1+r), 0)$ is the mirror image, yielding the characteristic \emph{bimodal} 
shape with peaks at $\pm 1/(1+r)$ and a trough at zero, as illustrated in 
Figure~\ref{fig:density_comparison}.

\section{Proof of Non-Emptiness and Measure-Zero Results}
\label{app:measure}

We now justify this design choice rigorously,
showing both that the feasible set $\mathcal{C}(G)$ is non-empty and
full-dimensional (in a suitable reduced parameterization), and that the
\emph{accidental} zeros it may introduce on edges are a negligible, measure-zero
phenomenon.

We collect here a sequence of results that together justify both
the convex relaxation~\eqref{eq:optimization_problem} adopted
throughout this paper and the claim that \emph{accidental} zeros on
graph edges are almost surely absent.
Before stating the results, we introduce all necessary definitions
and explain the rationale for each modelling choice.

\textit{Remark on notation.} The set $\mathcal{C}(G)$ defined in
\eqref{eq:omega_G} is a subset of the matrix space
$\mathcal{S}^p$, where the space of $p \times p$ symmetric matrices $\mathcal{S}^p$, 
identified with $\mathbb{R}^{p(p+1)/2}$ via the upper-triangular vectorization map, because this set is confined to a
proper affine subspace of $\mathcal{S}^p$ (fixed diagonal, fixed zeros), all
topological notions of interior and all Lebesgue-measure statements are
\emph{trivial} if stated directly in $\mathcal{S}^p$. Every point of
$\mathcal{C}(G)$ has empty matrix-space interior, and $\mathcal{C}(G)$ itself
has zero Lebesgue measure in $\mathcal{S}^p$. To make these notions
non-trivial, we introduce below a reduced, free-entry parameterization of
$\mathcal{C}(G)$ in a lower-dimensional Euclidean space, and we are careful to
use \emph{distinct} symbols for the matrix-space set $\mathcal{C}(G)$ and its
image under this parameterization, to avoid any ambiguity between the two.

\subsection{Ambient Space and the Feasible Set}


Fix a graph $G=(V,E)$ with $|V|=p$ and $|E|=q$, where $q$ counts
unordered pairs $\{i,j\}$ with $i<j$ such that $(i,j)\in E$.

A symmetric $p\times p$ matrix with unit diagonal and prescribed
zeros on all non-edges is \emph{completely} determined by the
values it takes on the edge set $E$.
Once we fix $c_{ii}=1$ for all $i$ and $c_{ij}=0$ for all
$(i,j)\notin E$, the only remaining freedom is the vector of
\emph{free entries}
\[
\bm{c} \;:=\; (c_{ij})_{(i,j)\in E,\, i<j}
\;\in\; \mathbb{R}^{q}.
\]
We define the \emph{reconstruction map} $\Phi_G : \mathbb{R}^{q} \to
\mathcal{S}^p$, which builds the full symmetric matrix from $\bm{c}$ by
setting the diagonal to $1$, filling each edge position $(i,j)$ and $(j,i)$
with $c_{ij}$, and filling each non-edge position with $0$. The map $\Phi_G$
is a \emph{linear bijection} from $\mathbb{R}^q$ onto the affine subspace
\begin{equation*}
\mathcal{A}_G := \bigl\{\mathbf{C}\in\mathcal{S}^p :
\mathrm{diag}(\mathbf{C})=\mathbf{1},\ c_{ij}=0\ \forall (i,j)\notin E\bigr\}
\subset \mathcal{S}^p,
\end{equation*}
which is precisely the affine subspace containing $\mathcal{C}(G)$. We use
$\Phi_G$ to transport $\mathcal{C}(G)$ into the lower-dimensional Euclidean
space $\mathbb{R}^q$, which (unlike $\mathcal{A}_G$, a proper affine subspace
of $\mathcal{S}^p$) carries a non-degenerate Lebesgue measure $\lambda_q$ and
a standard topology. We define the \emph{free-entry image} of $\mathcal{C}(G)$ as
\begin{equation}
\widehat{\mathcal{C}}(G)
\;:=\;
\Phi_G^{-1}\bigl(\mathcal{C}(G)\bigr)
\;=\;
\bigl\{\, \bm{c}\in\mathbb{R}^{q} :
\Phi_G(\bm{c})\succeq 0 \,\bigr\}
\;\subset\;\mathbb{R}^q,
\label{eq:feasible_image}
\end{equation}
where the equality on the right uses the fact that the unit-diagonal and
sparsity constraints are already built into $\Phi_G(\cdot)$, so the only
remaining constraint defining $\mathcal{C}(G)$ inside $\mathcal{A}_G$ is
positive semidefiniteness. \emph{All topological and measure-theoretic
statements in the remainder of this section (interior, Lebesgue measure,
dimension, density) refer to $\widehat{\mathcal{C}}(G)$ in $\mathbb{R}^q$, and
never to $\mathcal{C}(G)$ itself}, which remains, as always, the matrix set of
Eq.~\eqref{eq:omega_G} used in problem~\eqref{eq:optimization_problem}.

\textit{Semi-algebraic structure.}
A subset $S\subseteq\mathbb{R}^{n}$ is \emph{semi-algebraic} if
it is a finite Boolean combination (unions, intersections,
complements) of sets of the form
$\{\bm{x}\in\mathbb{R}^n:p(\bm{x})\geq 0\}$ for polynomials $p$
with real coefficients.
Positive semidefiniteness of $\Phi_G(\bm{c})$ is equivalent
to the non-negativity of all $2^p - 1$ leading principal minors,
each of which is a polynomial in $\bm{c}$.
The feasible set~\eqref{eq:feasible_image} is therefore a finite
intersection of polynomial sublevel sets, i.e.\ a \emph{closed
semi-algebraic} subset of $\mathbb{R}^{q}$.

\subsection{Topological Prerequisites}

We recall three standard notions used throughout the proofs.

\begin{definition}[Interior and relative interior]
\label{def:interior_relint}
Let $S\subseteq\mathbb{R}^{n}$.
\begin{itemize}
    \item The \emph{interior} of $S$, written $\mathrm{int}(S)$,
    is the largest open subset of $\mathbb{R}^n$ contained in
    $S$; equivalently, $\bm{x}\in\mathrm{int}(S)$ if and only
    if $\exists\,\varepsilon>0$ such that the open Euclidean
    ball $B(\bm{x},\varepsilon)\subseteq S$.
    \item The \emph{affine hull} of $S$, written $\mathrm{aff}(S)$,
    is the smallest affine subspace of $\mathbb{R}^n$ that
    contains $S$.
    \item The \emph{relative interior} of $S$, written
    $\mathrm{relint}(S)$, is the interior of $S$ taken with
    respect to the subspace topology of $\mathrm{aff}(S)$:
    $\bm{x}\in\mathrm{relint}(S)$ iff
    $\exists\,\varepsilon>0$ such that
    $B(\bm{x},\varepsilon)\cap\mathrm{aff}(S)\subseteq S$.
\end{itemize}
\end{definition}

\textit{Illustrative example.}
Let $T=\{(x,0):-1\leq x\leq 1\}\subset\mathbb{R}^2$ be a line
segment on the $x$-axis.
Its interior in $\mathbb{R}^2$ is \emph{empty}. No open disk of
$\mathbb{R}^2$ is contained in $T$.
Its relative interior in its affine hull (the $x$-axis
$\cong\mathbb{R}^1$) is the open segment $(-1,1)$.

The distinction matters here because $\mathcal{C}(G)$, viewed inside
the full matrix space $\mathcal{S}^p\cong\mathbb{R}^{p(p+1)/2}$,
has empty interior there (the unit-diagonal and sparsity constraints
confine it to the proper affine subspace $\mathcal{A}_G$, as noted above).
Once we pass to the reduced space $\mathbb{R}^q$ via $\Phi_G$,
the affine hull of $\widehat{\mathcal{C}}(G)$ \emph{is}
$\mathbb{R}^q$, so relative interior and interior coincide in that
space.
All interior claims below refer to $\mathbb{R}^q$.

\begin{definition}[Face of a convex set]
\label{def:face_convex}
Let $K\subseteq\mathbb{R}^n$ be a convex set.
A convex subset $F\subseteq K$ is a \emph{face} of $K$ if,
whenever $\bm{x}=(1-t)\bm{y}+t\bm{z}$ for some $\bm{y},\bm{z}\in K$
and $t\in(0,1)$, then both $\bm{y}$ and $\bm{z}$ belong to $F$.
Equivalently, no interior point of a segment in $K$ can lie in
$F$ without both endpoints also lying in $F$.
\end{definition}

Faces arise naturally as supporting-hyperplane intersections, if
$\bm{a}^\top\bm{x}\leq b$ holds for all $\bm{x}\in K$, then
$F=K\cap\{\bm{x}:\bm{a}^\top\bm{x}=b\}$ is a face.

\begin{definition}[Dimension and codimension]
\label{def:codimension}
The \emph{dimension} of an affine subspace $W\subseteq\mathbb{R}^n$
is the dimension of its direction space.
The \emph{codimension} of $W$ in $\mathbb{R}^n$ is
$\mathrm{codim}(W):=n-\dim(W)$.
A subspace of codimension $1$ is a \emph{hyperplane}.
\end{definition}

\textit{Illustrative example.}
In $\mathbb{R}^3$, a point has codimension $3$; a line has
codimension $2$; a plane has codimension $1$ and is therefore a
hyperplane.

\textit{Why codimension $1$ implies Lebesgue measure zero.}
A hyperplane $W$ in $\mathbb{R}^q$ has dimension $q-1$ and
``misses one dimension''. Fixing one coordinate to a constant
removes one degree of freedom, leaving zero $q$-dimensional
volume.
Formally, let $W=\ker(\pi_{ij})$ for a coordinate projection
$\pi_{ij}(\bm{c})=c_{ij}$.
By Fubini's theorem, decomposing $\mathbb{R}^q$ as
$\mathbb{R}^{q-1}\times\mathbb{R}$ along the $c_{ij}$-axis,
\[
\lambda_q(W)
\;=\;
\int_{-\infty}^{\infty}
\lambda_{q-1}(W\cap\{c_{ij}=t\})\,\mathbf{1}_{\{t=0\}}\,\mathrm{d}t
\;=\;
0,
\]
since the indicator $\mathbf{1}_{\{t=0\}}$ is non-zero only on
the single point $t=0$, a set of one-dimensional measure zero.
Hence every proper affine subspace of $\mathbb{R}^q$ is a
$\lambda_q$-null set.

\subsection{Non-Emptiness of the Feasible Set}

\begin{proposition}[Universal feasible point]
\label{prop:feasible}
For every graph $G=(V,E)$, the identity matrix satisfies
$\mathbf{I}_p\in\mathcal{C}(G)$.
\end{proposition}

\begin{proof}
Set $\bm{c}=\bm{0}\in\mathbb{R}^{q}$; the reconstructed matrix
is $\Phi_G(\bm{0})=\mathbf{I}_p$.
We verify each constraint in definition~\eqref{eq:omega_G}:
\begin{enumerate}
    \item \emph{Positive semidefiniteness:}
    $\mathbf{I}_p\succeq 0$ because all eigenvalues equal $1>0$.
    In fact $\mathbf{I}_p\succ 0$ (strictly positive definite).
    \item \emph{Unit diagonal:}
    $[\mathbf{I}_p]_{ii}=1$ for all $i$.
    \item \emph{Sparsity:}
    $[\mathbf{I}_p]_{ij}=0$ for every $i\neq j$, hence in
    particular for every non-edge $(i,j)\notin E$.
\end{enumerate}
All constraints hold for any graph $G$.
\end{proof}

\begin{remark}[The origin is an interior point of $\widehat{\mathcal{C}}(G)$]
\label{rem:interior_point}
A stronger statement holds, $\bm{0}\in\mathrm{int}(\widehat{\mathcal{C}}(G))$
in $\mathbb{R}^q$.

Strict positive definiteness is an \emph{open} condition. The
minimum-eigenvalue function
$\bm{c}\mapsto\lambda_{\min}(\Phi_G(\bm{c}))$
is continuous (eigenvalues vary continuously with the matrix
entries, which are linear in $\bm{c}$), so
$\{\bm{c}:\lambda_{\min}(\Phi_G(\bm{c}))>0\}$ is the preimage
of the open set $(0,\infty)$ under a continuous map and is
therefore open in $\mathbb{R}^q$.
Since $\lambda_{\min}(\mathbf{I}_p)=1>0$, there exists
$\varepsilon>0$ such that $\Phi_G(\bm{c})\succ 0$
(hence $\bm{c}\in\widehat{\mathcal{C}}(G)$) for all $\|\bm{c}\|_2\leq\varepsilon$.
Thus $B(\bm{0},\varepsilon)\subseteq\widehat{\mathcal{C}}(G)$, confirming
$\bm{0}\in\mathrm{int}(\widehat{\mathcal{C}}(G))$.

Because $\mathrm{int}(\widehat{\mathcal{C}}(G))\neq\emptyset$, the Lebesgue
measure satisfies $\lambda_q(\widehat{\mathcal{C}}(G))>0$, the free-entry
feasible set is a \emph{full-dimensional} convex body in $\mathbb{R}^{q}$.
(Equivalently, $\mathcal{C}(G)$ is full-dimensional \emph{relative to} the
affine subspace $\mathcal{A}_G$ it lives in, in the sense of
$\mathrm{relint}$.)
\end{remark}

\subsection{Lebesgue Measure of the \emph{Accidental Zero} Set}

Define the \emph{accidental-zero set} as
\begin{equation}
\widehat{\mathcal{N}}(G)
\;:=\;
\Bigl\{\,\bm{c}\in\widehat{\mathcal{C}}(G)
: \exists\,(i,j)\in E \text{ s.t.\ } c_{ij}=0\Bigr\}
\;\subset\;\mathbb{R}^q,
\label{eq:accidental_zero}
\end{equation}
i.e.\ the free-entry vectors whose reconstructed matrix has at least one
zero edge entry.

\begin{proposition}[Accidental zeros form a null set]
\label{prop:measure_zero}
$\lambda_q\bigl(\widehat{\mathcal{N}}(G)\bigr)=0$.
\end{proposition}

\begin{proof}
For each edge $(i,j)\in E$ ($i<j$) let
\[
H_{ij}
\;:=\;
\bigl\{\bm{c}\in\mathbb{R}^{q} : c_{ij}=0\bigr\}
\;=\;
\ker(\pi_{ij}),
\]
where $\pi_{ij}(\bm{c})=c_{ij}$ is the $ij$-coordinate
projection.
Because $\pi_{ij}$ is a non-zero surjective linear functional,
$H_{ij}$ is a linear subspace of dimension $q-1$, i.e.\ of
codimension $1$ in $\mathbb{R}^q$
(Definition~\ref{def:codimension}).
By the Fubini argument above, $\lambda_q(H_{ij})=0$.

The accidental-zero set satisfies
\begin{equation}
\widehat{\mathcal{N}}(G)
\;\subseteq\;
\bigcup_{(i,j)\in E}\bigl(H_{ij}\cap\widehat{\mathcal{C}}(G)\bigr).
\label{eq:union_decomp}
\end{equation}
Each term is a subset of the null set $H_{ij}$, so
$\lambda_q(H_{ij}\cap\widehat{\mathcal{C}}(G))=0$.
Finite sub-additivity of $\lambda_q$ ($|E|=q<\infty$) gives
\[
\lambda_q\bigl(\widehat{\mathcal{N}}(G)\bigr)
\;\leq\;
\sum_{(i,j)\in E}\lambda_q\bigl(H_{ij}\cap\widehat{\mathcal{C}}(G)\bigr)
\;=\;
0.
\qedhere
\]
\end{proof}

\begin{remark}[Why $H_{ij}\cap\widehat{\mathcal{C}}(G)$ is not a face]
\label{rem:not_a_face}
Each intersection $H_{ij}\cap\widehat{\mathcal{C}}(G)$ is a codimension-$1$
cross-section of the convex body $\widehat{\mathcal{C}}(G)$.
It is \emph{not} in general a face of $\widehat{\mathcal{C}}(G)$ in the
sense of Definition~\ref{def:face_convex}.
A face requires the supporting hyperplane to leave $\widehat{\mathcal{C}}(G)$
entirely on one side; since $\bm{0}\in\mathrm{int}(\widehat{\mathcal{C}}(G))$
(Remark~\ref{rem:interior_point}) and $c_{ij}=0$ at $\bm{0}$,
there exist points in $\widehat{\mathcal{C}}(G)$ with $c_{ij}>0$ and points
with $c_{ij}<0$.
Hence $H_{ij}$ cuts \emph{through the interior} of $\widehat{\mathcal{C}}(G)$
rather than merely touching its boundary, and the measure-zero
conclusion rests entirely on the codimension argument.
\end{remark}

\begin{remark}[Density and almost-sure behaviour]
\label{rem:density}
By Proposition~\ref{prop:measure_zero} and
Remark~\ref{rem:interior_point}, the complement
$\widehat{\mathcal{C}}(G)\setminus\widehat{\mathcal{N}}(G)$ is dense in
$\widehat{\mathcal{C}}(G)$. Almost every free-entry vector has strictly
non-zero values on all edge entries.
In simulations where $\bar{\mathbf{C}}$ is drawn from any
$\lambda_q$-absolutely continuous distribution (in the free-entry
parameterization), the minimiser of
\eqref{eq:optimization_problem} almost surely avoids $\widehat{\mathcal{N}}(G)$,
i.e.\ almost surely $\hat{\mathbf{C}} \notin \mathcal{N}(G) :=
\Phi_G(\widehat{\mathcal{N}}(G))$, the corresponding matrix-space
accidental-zero set.
\end{remark}

\subsection{The Non-Convex Alternative and Its Properties}

One might seek to enforce the stronger \emph{if-and-only-if} zero-pattern by solving the modified optimization problem:
\begin{equation}
    \begin{aligned}
        &\min_{\mathbf{C}} && \frac{1}{2} \, \|\mathbf{C}-\bar{\mathbf{C}}\|_F^2 \\
        &\text{subject to} && \mathbf{C} \in \mathcal{C}^+(G),
    \end{aligned}
    \tag{P$^+$}
    \label{eq:nonconvex_problem}
\end{equation}
where the strict feasible set is defined as
\begin{equation}
    \mathcal{C}^+(G) \coloneqq \bigl\{ \mathbf{C}\in\mathcal{C}(G) : c_{ij}\neq 0 , \; \forall (i,j)\in E \bigr\} = \mathcal{C}(G) \setminus \mathcal{N}(G),
    \label{eq:strict_feasible}
\end{equation}
with the corresponding free-entry image $\widehat{\mathcal{C}}^+(G) \coloneqq \Phi_G^{-1}\bigl(\mathcal{C}^+(G)\bigr) = \widehat{\mathcal{C}}(G) \setminus \widehat{\mathcal{N}}(G)$.

\begin{definition}[Coercivity]
\label{def:coercive}
A function $f:\mathbb{R}^n\to\mathbb{R}$ is \emph{coercive} if
$f(\bm{x})\to+\infty$ as $\|\bm{x}\|_2\to\infty$, i.e.\ every
sublevel set $\{\bm{x}:f(\bm{x})\leq\alpha\}$ is bounded.
Coercivity together with continuity guarantees that $f$ attains
its infimum on any closed (not necessarily compact) subset of
$\mathbb{R}^n$, by preventing minimising sequences from escaping to
infinity.
In the present setting $\widehat{\mathcal{C}}(G)$ is already compact (closed
and bounded, since $|c_{ij}|\leq 1$), so coercivity holds
automatically; we include the definition for completeness.
\end{definition}

\begin{proposition}[Properties of $\widehat{\mathcal{C}}^+(G)$]
\label{prop:nonconvex}
~
\begin{enumerate}
    \item[\emph{(i)}]   $\widehat{\mathcal{C}}^+(G)$ is \emph{non-empty}.
    \item[\emph{(ii)}]  $\widehat{\mathcal{C}}^+(G)$ is \emph{open} in the
    relative topology of $\widehat{\mathcal{C}}(G)$, hence \emph{not
    closed} in $\mathbb{R}^q$.
    \item[\emph{(iii)}] $\widehat{\mathcal{C}}^+(G)$ is \emph{non-convex}.
    \item[\emph{(iv)}]  Problem~\eqref{eq:nonconvex_problem} admits
    a \emph{minimiser}, but \emph{not necessarily a unique one}.
    \item[\emph{(v)}]   $\displaystyle
    \inf_{\mathbf{C}\in\mathcal{C}^+(G)}
    \tfrac{1}{2}\|\mathbf{C}-\bar{\mathbf{C}}\|_F^2
    \;=\;
    \min_{\mathbf{C}\in\mathcal{C}(G)}
    \tfrac{1}{2}\|\mathbf{C}-\bar{\mathbf{C}}\|_F^2$.
\end{enumerate}
\end{proposition}

\begin{proof}
\emph{(i) Non-emptiness.}
By Remark~\ref{rem:interior_point}, there exists $\varepsilon>0$ such that
$B(\bm{0},\varepsilon)\subseteq\widehat{\mathcal{C}}(G)$. Consider the point
$\bm{c}_0 := \dfrac{\varepsilon}{2\sqrt{q}}\,\bm{1}_q$, i.e.\ every free entry
set to the common value $\delta := \varepsilon/(2\sqrt{q}) \neq 0$. Then
$\|\bm{c}_0\|_2 = \varepsilon/2 < \varepsilon$, so
$\bm{c}_0\in B(\bm{0},\varepsilon)\subseteq\widehat{\mathcal{C}}(G)$, and by
construction $[\bm{c}_0]_{ij} = \delta \neq 0$ for every edge $(i,j)\in E$.
Hence $\bm{c}_0\in\widehat{\mathcal{C}}^+(G)$, so
$\widehat{\mathcal{C}}^+(G)\neq\emptyset$, and therefore $\mathcal{C}^+(G) =
\Phi_G(\widehat{\mathcal{C}}^+(G)) \neq \emptyset$.

\smallskip
\emph{(ii) Openness and non-closedness.}
Each $H_{ij}\cap\widehat{\mathcal{C}}(G)$ is closed in $\widehat{\mathcal{C}}(G)$
(intersection of two closed sets), so
$\widehat{\mathcal{N}}(G)=\widehat{\mathcal{C}}(G)\cap\bigcup_{(i,j)\in E}H_{ij}$
is closed in $\widehat{\mathcal{C}}(G)$ (finite union of closed sets).
Its complement
$\widehat{\mathcal{C}}^+(G)=\widehat{\mathcal{C}}(G)\setminus\widehat{\mathcal{N}}(G)$
is therefore relatively open in $\widehat{\mathcal{C}}(G)$.
Non-closedness, $\bm{0}\in\widehat{\mathcal{N}}(G)$ (all free entries
of $\Phi_G(\bm{0})=\mathbf{I}_p$ are zero), yet $\bm{0}$ is a limit of
points in $\widehat{\mathcal{C}}^+(G)$, so
$\bm{0}\in\overline{\widehat{\mathcal{C}}^+(G)}\setminus\widehat{\mathcal{C}}^+(G)$.

\smallskip
\emph{(iii) Non-convexity.}
Fix any edge $(i,j)\in E$.
By openness of $\mathrm{int}(\widehat{\mathcal{C}}(G))$ and the intermediate
value theorem, there exist points
$\bm{c}_1,\bm{c}_2\in\mathrm{int}(\widehat{\mathcal{C}}(G))
\subseteq\widehat{\mathcal{C}}^+(G)$
with $[\bm{c}_1]_{ij}>0$ and $[\bm{c}_2]_{ij}<0$.
(Formally, start from any interior point, perturb the
$c_{ij}$-coordinate in each direction while staying inside the
open interior.)
The convex combination $\bm{c}_\lambda=(1-\lambda)\bm{c}_1
+\lambda\bm{c}_2$ satisfies $[\bm{c}_\lambda]_{ij}=0$ at
\[
\lambda^*
=
\frac{[\bm{c}_1]_{ij}}{[\bm{c}_1]_{ij}-[\bm{c}_2]_{ij}}
\in(0,1).
\]
Since $\widehat{\mathcal{C}}(G)$ is convex (it is the preimage of the convex
set $\mathcal{C}(G)$ under the linear map $\Phi_G$), $\bm{c}_{\lambda^*}\in
\widehat{\mathcal{C}}(G)$;
but $[\bm{c}_{\lambda^*}]_{ij}=0$, so
$\bm{c}_{\lambda^*}\in\widehat{\mathcal{N}}(G)$, hence
$\bm{c}_{\lambda^*}\notin\widehat{\mathcal{C}}^+(G)$.

\smallskip
\emph{(iv) Existence but non-uniqueness.}
The objective $f(\bm{c})=\tfrac{1}{2}\|\Phi_G(\bm{c})
-\bar{\mathbf{C}}\|_F^2$ is continuous and (Definition~\ref{def:coercive})
coercive on $\mathbb{R}^q$; $\widehat{\mathcal{C}}(G)$ is compact; so $f$
attains its minimum $f^*$ at a unique point
$\bm{c}^*\in\widehat{\mathcal{C}}(G)$ (uniqueness from strict convexity of
$f$ over the convex set $\widehat{\mathcal{C}}(G)$, as established earlier
in this section).
By Remark~\ref{rem:density}, $\widehat{\mathcal{C}}^+(G)$ is dense in
$\widehat{\mathcal{C}}(G)$, so for every $\varepsilon>0$ there exists
$\bm{c}_\varepsilon\in\widehat{\mathcal{C}}^+(G)$ with
$f(\bm{c}_\varepsilon)<f^*+\varepsilon$.
Hence $\inf_{\widehat{\mathcal{C}}^+(G)}f=f^*$.
If $\bm{c}^*\in\widehat{\mathcal{C}}^+(G)$ (which occurs almost surely by
Proposition~\ref{prop:measure_zero}), it is also the minimiser
of~\eqref{eq:nonconvex_problem}.
Uniqueness can fail because $\widehat{\mathcal{C}}^+(G)$ is not convex:
the standard argument (strictly convex objective $\Rightarrow$
unique minimiser) requires convexity of the feasible set, which
fails here.
If $\bm{c}^*\in\widehat{\mathcal{N}}(G)$, the infimum is not attained in
$\widehat{\mathcal{C}}^+(G)$ at all, and any minimising sequence converges
to a point outside $\widehat{\mathcal{C}}^+(G)$.

\smallskip
\emph{(v) Equality of infima.}
$\inf_{\mathcal{C}^+(G)}f\geq\min_{\mathcal{C}(G)}f=f^*$
is immediate from $\mathcal{C}^+(G)\subseteq\mathcal{C}(G)$.
The reverse follows from the density argument in (iv):
$\inf_{\mathcal{C}^+(G)}f\leq f(\bm{c}_\varepsilon)\to f^*$.
\end{proof}

\begin{remark}[Practical consequences]
\label{rem:practical}
Parts (iii)--(v) of Proposition~\ref{prop:nonconvex} jointly
establish that the convex relaxation~\eqref{eq:optimization_problem}
and the non-convex alternative~\eqref{eq:nonconvex_problem} share
the same optimal value $f^*$, but only the convex problem is
guaranteed to deliver a \emph{unique}, \emph{efficiently computable}
minimiser.
The non-convexity of $\mathcal{C}^+(G)$ invalidates standard
first-order optimality conditions as sufficient criteria for global
optimality, and solving~\eqref{eq:nonconvex_problem} directly would
require global optimisation strategies (e.g.\ branch-and-bound),
forfeiting the polynomial-time convergence guarantees of the three
solvers presented in Section~\ref{sec:solvers}.
We therefore adopt~\eqref{eq:optimization_problem} throughout,
accepting the theoretical possibility of accidental edge zeros
which, by Proposition~\ref{prop:measure_zero}, occurs with
$\lambda_q$-probability zero in any simulation scenario.
\end{remark}

\section{Radon--Nikodym Decomposition of the Projected Correlation Density}
\label{app:radon_nikodym}

We now analyze the distribution of projected solutions 
$\bm{c}^* = \pi_{\mathcal{F}}(\bar{\bm{c}})$ 
rigorously. The starting point is the 
\emph{Radon--Nikod\'{y}m theorem}. Given two 
$\sigma$-finite measures $\nu$ and $\lambda$ on a 
measurable space, if $\nu$ is absolutely continuous 
with respect to $\lambda$ (written $\nu \ll \lambda$), 
there exists a unique $\lambda$-a.e.\ measurable 
function $d\nu/d\lambda$, called the 
Radon--Nikod\'{y}m derivative, such that 
$\nu(A) = \int_A (d\nu/d\lambda)\,d\lambda$ for every 
measurable set $A$. When $\nu$ is \emph{not} absolutely 
continuous with respect to $\lambda$, the theorem 
guarantees instead the Lebesgue decomposition 
$\nu = \nu_{\mathrm{ac}} + \nu_{\mathrm{sing}}$, where 
$\nu_{\mathrm{ac}} \ll \lambda$ admits a 
Radon--Nikod\'{y}m derivative, and $\nu_{\mathrm{sing}}$ 
is singular with respect to $\lambda$ 
($\nu_{\mathrm{sing}} \perp \lambda$), meaning it is 
concentrated on a set of $\lambda$-measure zero. 
We will show that the projection induced by 
problem~\eqref{eq:optimization_problem_with_mean} 
produces exactly such a decomposition, with a singular 
part concentrated on the visible ring of 
Figure~\ref{fig:our_method}.

We first establish that solving 
problem~\eqref{eq:optimization_problem_with_mean} 
without a sparsity constraint (i.e., $G$ is the 
complete graph) is equivalent to computing the 
orthogonal projection of $\bar{\bm{c}}$ onto 
$\mathcal{F}$. In this setting the feasible set reduces 
to
\begin{equation}
\mathcal{F} = \mathcal{C} \cap \mathcal{H}^+, \qquad
\mathcal{H}^+ = \bigl\{\bm{c}\in\mathbb{R}^3 : 
\bm{a}^\top\bm{c} \geq 3t\bigr\},
\end{equation}
where $\bm{a} = (1,1,1)^\top$ and 
$\bm{c} = (c_{12},c_{13},c_{23})^\top$. Since 
$\mathcal{F}$ is a non-empty closed convex set and 
the objective $\frac{1}{2}\|\bm{c} - 
\bar{\bm{c}}\|^2$ is strictly convex, the unique 
minimizer is the orthogonal projection 
$\bm{c}^* = \pi_{\mathcal{F}}(\bar{\bm{c}})$, given 
explicitly by
\begin{equation}
\pi_{\mathcal{F}}(\bar{\bm{c}}) =
\begin{cases}
\bar{\bm{c}} 
& \text{if } \bm{a}^\top\bar{\bm{c}} \geq 3t,\\[4pt]
\bar{\bm{c}} + \dfrac{3t - 
\bm{a}^\top\bar{\bm{c}}}{3}\,\bm{1} 
& \text{if } \bm{a}^\top\bar{\bm{c}} < 3t,
\end{cases}
\label{eq:proj_map_full}
\end{equation}
where the second case follows from the fact that the 
nearest point on the hyperplane 
$\mathcal{H} = \{\bm{c} : \bm{a}^\top\bm{c} = 3t\}$ 
to any $\bar{\bm{c}} \notin \mathcal{H}^+$ is obtained 
by moving $\bar{\bm{c}}$ along the normal direction 
$\bm{a}/\|\bm{a}\|^2 = \bm{1}/3$ by the signed 
distance $(3t - \bm{a}^\top\bar{\bm{c}})/3 > 0$.

We now link this projection to the Radon--Nikod\'{y}m 
framework. Let $\mu$ denote the uniform probability 
measure on $\mathcal{C}$, absolutely continuous with 
respect to the three-dimensional Lebesgue measure 
$\lambda_3$, with constant Radon--Nikod\'{y}m 
derivative
\begin{equation}
\frac{d\mu}{d\lambda_3}(\bar{\bm{c}}) 
= \frac{1}{\mathrm{vol}(\mathcal{C})}, 
\quad \bar{\bm{c}}\in\mathcal{C}.
\end{equation}
The distribution of projected solutions is the 
pushforward measure $\nu := (\pi_{\mathcal{F}})_\#\mu$, 
defined by $\nu(A) = \mu(\pi_{\mathcal{F}}^{-1}(A))$ 
for every Borel set $A \subseteq \mathcal{F}$. 
The domain $\mathcal{C}$ splits as 
$\mathcal{C} = \mathcal{C}^+ \cup \mathcal{C}^-$, where
\begin{equation}
\mathcal{C}^+ = \mathcal{C}\cap\mathcal{H}^+, \qquad
\mathcal{C}^- = \mathcal{C}\setminus\mathcal{C}^+.
\end{equation}
On $\mathcal{C}^+$ the projection is the identity, so 
the pushforward of $\mu|_{\mathcal{C}^+}$ is absolutely 
continuous with respect to $\lambda_3$, with 
Radon--Nikod\'{y}m derivative
\begin{equation}
\frac{d\nu_{\mathrm{ac}}}{d\lambda_3}(\bm{c}^*) 
= \frac{1}{\mathrm{vol}(\mathcal{C})}, 
\quad \bm{c}^*\in\mathrm{int}(\mathcal{C}^+).
\label{eq:density_ac}
\end{equation}
On $\mathcal{C}^-$, however, every point is mapped by 
\eqref{eq:proj_map_full} onto the two-dimensional 
manifold $\Gamma := \mathcal{C}\cap\mathcal{H}$, which 
satisfies $\lambda_3(\Gamma) = 0$. The pushforward of 
$\mu|_{\mathcal{C}^-}$ therefore concentrates a mass 
of $\mu(\mathcal{C}^-) = 
\mathrm{vol}(\mathcal{C}^-)/\mathrm{vol}(\mathcal{C}) 
> 0$ onto a set of $\lambda_3$-measure zero, making it 
singular with respect to $\lambda_3$. By the Lebesgue 
decomposition theorem, the total pushforward splits as
\begin{equation}
\nu = \nu_{\mathrm{ac}} + \nu_{\mathrm{sing}},
\label{eq:nu_decomp}
\end{equation}
where $\nu_{\mathrm{ac}} \ll \lambda_3$ has constant 
density~\eqref{eq:density_ac} on 
$\mathrm{int}(\mathcal{C}^+)$, and 
$\nu_{\mathrm{sing}} \perp \lambda_3$ is supported on 
$\Gamma$. To characterize the density of $\nu_{\mathrm{sing}}$ 
on $\Gamma$, we use a slicing argument. Decompose 
$\mathbb{R}^3$ into the normal direction 
$\bm{n} = \bm{1}/\sqrt{3}$ to $\mathcal{H}$ and the 
parallel directions within $\mathcal{H}$. For any 
$\bm{c}^* \in \Gamma$, the pre-image 
$\pi_{\mathcal{F}}^{-1}(\bm{c}^*)$ under 
\eqref{eq:proj_map_full} is the line segment
\begin{equation}
\pi_{\mathcal{F}}^{-1}(\bm{c}^*) = 
\bigl\{\bm{c}^* - \delta\bm{1} : 
0 \leq \delta \leq \ell(\bm{c}^*)\bigr\} 
\subset \mathcal{C}^-,
\label{eq:preimage_segment}
\end{equation}
where $\ell(\bm{c}^*) = \sup\{\delta > 0 : 
\bm{c}^* - \delta\bm{1} \in \mathcal{C}\}$ is the 
length of the chord of $\mathcal{C}$ in the direction 
$-\bm{1}$ starting from $\bm{c}^*$, i.e.\ the largest 
$\delta$ for which the shifted point 
$\bm{c}^* - \delta\bm{1} = 
(c_{12}^*-\delta,\, c_{13}^*-\delta,\, c_{23}^*-\delta)$ 
remains inside the elliptope, equivalently satisfies
\begin{equation}
    1 - (c_{12}^*-\delta)^2 - (c_{13}^*-\delta)^2 - (c_{23}^*-\delta)^2 + 2(c_{12}^*-\delta)(c_{13}^*-\delta)(c_{23}^*-\delta) \geq 0.
    \label{eq:elliptope_shifted}
\end{equation}
Since $\mu$ 
has constant density $1/\mathrm{vol}(\mathcal{C})$ 
with respect to $\lambda_3$, the mass assigned by 
$\mu$ to this segment is proportional to its length:
\begin{equation}
\mu\bigl(\pi_{\mathcal{F}}^{-1}(\bm{c}^*)\bigr) 
\propto \frac{\ell(\bm{c}^*)\sqrt{3}}
{\mathrm{vol}(\mathcal{C})},
\label{eq:preimage_mass}
\end{equation}
where the factor $\sqrt{3} = \|\bm{1}\|$ accounts for 
the parametrization of the segment by $\delta$. 
By definition of the pushforward, the density of 
$\nu_{\mathrm{sing}}$ at $\bm{c}^* \in \Gamma$ is 
therefore proportional to $\ell(\bm{c}^*)$. Points 
on $\Gamma$ with longer pre-image segments in 
$\mathcal{C}^-$ accumulate more projected mass and 
appear denser in Figure~\ref{fig:our_method}.

Equations~\eqref{eq:density_ac} and 
\eqref{eq:preimage_mass} together provide a complete 
characterization of the density of projected solutions 
across all regions of $\mathcal{F}$, and directly 
explain the visual structure of 
Figure~\ref{fig:our_method}. Interior points of 
$\mathcal{C}^+$ inherit the uniform density of $\mu$ 
and appear sparsely distributed, while the ring 
$\Gamma$ accumulates the entire mass 
$\mu(\mathcal{C}^-) > 0$ onto a set of 
$\lambda_3$-measure zero. To make this density contrast quantitatively precise, 
consider small balls of radius $\varepsilon > 0$ 
centered at a ring point 
$\bm{c}^*_{\mathrm{ring}} \in \Gamma$ and an interior 
point $\bm{c}^*_{\mathrm{in}} \in 
\mathrm{int}(\mathcal{C}^+)$ respectively. 

At the \emph{interior point}, $\nu$ coincides with 
$\nu_{\mathrm{ac}}$, which has constant density 
$1/\mathrm{vol}(\mathcal{C})$ with respect to 
$\lambda_3$. The ball $B_\varepsilon
(\bm{c}^*_{\mathrm{in}})$ is a full three-dimensional 
ball of volume $\frac{4}{3}\pi\varepsilon^3$, so
\begin{equation}
\nu\bigl(B_\varepsilon(\bm{c}^*_{\mathrm{in}})\bigr) 
\;\sim\; 
\frac{1}{\mathrm{vol}(\mathcal{C})} \cdot 
\frac{4}{3}\pi\varepsilon^3.
\label{eq:mass_interior}
\end{equation}

At the \emph{ring point}, $\nu$ coincides with 
$\nu_{\mathrm{sing}}$, which is supported entirely on 
the two-dimensional manifold $\Gamma$. The ball 
$B_\varepsilon(\bm{c}^*_{\mathrm{ring}})$ therefore 
intersects $\Gamma$ in a two-dimensional disk of area 
$\sim \pi\varepsilon^2$, and the total mass received 
by this disk is, by \eqref{eq:preimage_mass}, 
proportional to $\ell(\bm{c}^*_{\mathrm{ring}})$ 
times this area:
\begin{equation}
\nu\bigl(B_\varepsilon(\bm{c}^*_{\mathrm{ring}})\bigr) 
\;\sim\; 
\frac{\ell(\bm{c}^*_{\mathrm{ring}})\sqrt{3}}
{\mathrm{vol}(\mathcal{C})} \cdot \pi\varepsilon^2.
\label{eq:mass_ring}
\end{equation}

Taking the ratio of \eqref{eq:mass_ring} and 
\eqref{eq:mass_interior} gives
\begin{equation}
\frac{\nu(B_\varepsilon(\bm{c}^*_{\mathrm{ring}}))}
{\nu(B_\varepsilon(\bm{c}^*_{\mathrm{in}}))}
\;\sim\;
\frac{\ell(\bm{c}^*_{\mathrm{ring}})
\sqrt{3}\cdot\varepsilon^2}
{\frac{4}{3}\varepsilon^3}
\;
\;\xrightarrow{\varepsilon\to 0}\;+\infty.
\label{eq:density_ratio}
\end{equation}
The ratio diverges as $\varepsilon \to 0$ because the 
numerator decays only as $\varepsilon^2$ (mass 
accumulating on a surface) while the denominator 
decays as $\varepsilon^3$ (mass spread through a 
volume). This rigorously confirms that the ring 
$\Gamma$ is \emph{infinitely denser} than any interior 
region of $\mathcal{F}$ at arbitrarily fine scales, 
explaining the striking visual concentration observed 
in Figure~\ref{fig:our_method}.

\bibliographystyle{unsrt}  
\bibliography{sample}

\end{document}